\documentclass{article}

\usepackage{times}
\usepackage{latexsym}

\usepackage[T1]{fontenc}

\usepackage[utf8]{inputenc}

\usepackage{microtype}

\usepackage{inconsolata}

\usepackage{graphicx}
\usepackage{placeins}

\usepackage{amsthm}
\theoremstyle{plain}
\newtheorem{theorem}{Theorem}[section]

\newtheorem{corollary}[theorem]{Corollary}
\theoremstyle{definition}
\newtheorem{definition}{Definition}
\newtheorem{proposition}{Proposition}
\newtheorem{property}{Property}

\newtheorem{hypothesis}{Hypothesis}
\theoremstyle{remark}

\usepackage[most]{tcolorbox}
\usepackage{multicol}

 \makeatletter
\def\@fnsymbol#1{\ensuremath{\ifcase#1\or \dagger\or \ddagger\or
  \mathsection\or \mathparagraph\or \|\or **\or \dagger\dagger
  \or \ddagger\ddagger \else\@ctrerr\fi}}
\makeatother

\usepackage{dsfont}

\usepackage{multirow}
\usepackage{amsthm,amssymb,lipsum}
\usepackage{mathrsfs}
\usepackage{mathtools}
\usepackage{enumerate}
\usepackage{colortbl}
\usepackage{wrapfig}
\usepackage{bbding}
\usepackage{pifont}
\usepackage{xspace}
\usepackage{amsmath}
\usepackage{wasysym}
\usepackage{textcomp}
\usepackage[bottom]{footmisc}

\usepackage{makecell}

\usepackage{color}
\usepackage{tocloft}
\usepackage{caption}
\usepackage{float}
\usepackage{amsmath}
\usepackage{listings}
\lstdefinestyle{mypython}{
  language=Python,
  basicstyle=\ttfamily\small,
  numbersep=8pt,
  tabsize=4,
  showstringspaces=false,
  keywordstyle=\color{blue},
  stringstyle=\color{teal!70!black},
  commentstyle=\color{gray},
  breaklines=true
}
\usepackage{afterpage}

\usepackage{stackengine}

\usepackage[commandnameprefix=always]{changes}

\usepackage{bbm}

\usepackage{duckuments}

\usepackage{tabularx}
\usepackage{booktabs}
\usepackage[percent]{overpic}
\usepackage{adjustbox}
\usepackage{utfsym}
\usepackage{thmtools,thm-restate}

\usepackage{epigraph}

\def\ie{{\it i.e.}\xspace}

\definecolor{drp-blue}{HTML}{1f77b4}
\definecolor{pretty-blue}{RGB}{0, 113, 188}
\definecolor{kaiming-green}{RGB}{57,181,74} 
\definecolor{mypurple}{RGB}{55,0,168} 
\definecolor{icmlblue}{rgb}{0,0.08,0.45} 
\definecolor{linecolor1}{HTML}{F1F7FB}
\definecolor{linecolor2}{HTML}{E3EFF7}
\definecolor{linecolor3}{HTML}{D5E4F0}

\definecolor{reconcolor}{HTML}{412F8A}
\definecolor{runpei-orange}{HTML}{F35F27}
\definecolor{runpei_blue}{HTML}{14294B}
\definecolor{datacolor}{HTML}{0009BF}
\definecolor{vitcolor}{HTML}{fc8e62}
\definecolor{xycolor}{HTML}{EF98AA}
\definecolor{lightpink}{HTML}{F9F0EE}
\definecolor{deeppink}{HTML}{EBD9E4}
\definecolor{lightpurple}{RGB}{123,107,143}
\definecolor{lightyellow}{RGB}{200,180,120}
\definecolor{MorandiLightBlue}{RGB}{180,200,240}
\definecolor{MorandiLighterBlue}{RGB}{240,245,255}
\definecolor{MorandiPink}{RGB}{200, 160, 180}
\definecolor{MorandiLightPink}{RGB}{250, 240, 245}

\definecolor{MorandiLightPink}{RGB}{250, 240, 245}
\definecolor{MorandiLightPink}{RGB}{250, 240, 245}

\usepackage{enumitem}

 \usepackage[main, final]{neurips_2026}

\usepackage[utf8]{inputenc} 
\usepackage[T1]{fontenc}    
\usepackage{url}            
\usepackage{booktabs}       
\usepackage{amsfonts}       
\usepackage{nicefrac}       
\usepackage{microtype}      
\usepackage[x11names,dvipsnames,table]{xcolor}
\definecolor{LinkBlue}{HTML}{0B63C3}
\usepackage[colorlinks=True,
            linkcolor=Maroon,
            anchorcolor=blue,  
            urlcolor=LinkBlue,
            citecolor=icmlblue,
            ]{hyperref}
\usepackage{graphicx}
\usepackage{subcaption}
\def\ours{\textsc{LoRe}}
\def\oursbench{\textsc{LoRe-Bench}}
\def\oursmonobench{\textsc{LoRe-Mono}}
\def\ourscompobench{\textsc{LoRe-Compo}}

\title{When Reasoning Meets Its Laws}

\author{%
 Junyu Zhang~$^{13*\ddagger}$
  \quad
  Yifan Sun~$^{1*}$
  \quad
  Tianang Leng~$^{3*}$
  \quad
  Jingyan Shen~$^{4*}$
  \\
  \textbf{Liu Ziyin~$^{25\dagger}$
  \quad
  Paul Pu Liang~$^{2\dagger}$
  \quad
  Huan Zhang~$^{1\dagger}$}
\\\\
$^{1}$University of Illinois Urbana-Champaign\quad
$^{2}$Massachusetts Institute of Technology\\
$^{3}$University of Pennsylvania\quad
$^{4}$New York University\quad
$^{5}$NTT Research
}

\begin{document}

\maketitle

\begin{abstract}
Despite the superior performance of Large Reasoning Models (LRMs), their reasoning behaviors are often counterintuitive, \textit{e.g.}, excessive reasoning on simple questions or insufficient reasoning on complex ones, leading to suboptimal reasoning capabilities. This paper presents the Laws of Reasoning (\ours), a unified framework that formalizes intrinsic reasoning patterns in ideal LRMs. We first propose \textit{compute law} with the hypothesis that the reasoning compute should scale linearly with question complexity. Beyond compute, we extend \ours\space with a supplementary \textit{accuracy law}, positing exponential decay of accuracy with increasing complexity. Since the complexity is difficult to quantify in practice, we approximate these hypotheses via two tractable properties, \textit{monotonicity} and \textit{compositionality}. We therefore introduce \oursbench, a benchmark that systematically measures these properties for large reasoning models. Evaluation shows that most reasoning models exhibit reasonable monotonicity but lack compositionality. In response, we develop an effective finetuning approach that enforces compute-law compositionality. Extensive empirical studies demonstrate that better compliance with compute laws yields consistently improved reasoning performance on multiple benchmarks, and uncovers synergistic gains across properties and laws. Project page: 
\href{lore-project.github.io}{\url{https://lore-project.github.io/}}.

{\let\thefootnote\relax
\footnotetext{\hspace*{-0.45em}$^*$Equal contribution. \quad
$^\dagger$Equal mentorship. \quad
$^\ddagger$Work partially done during the internship at MIT.\\ 
\hspace*{1.9em}Correspondence to: \texttt{junyuz6@engineering.upenn.edu, ppliang@mit.edu, huan@huan-zhang.com}.}
}
\end{abstract}

\vspace{-0.2cm}
\section{Introduction}
\vspace{-0.1cm}

Large Reasoning Models (LRMs) such as OpenAI o1~\citep{OpenAIO1_24} have demonstrated unprecedented progress in approaching human-like reasoning capabilities. Despite their strong performance on solving complex problems, even powerful LRMs exhibit \textit{abnormal behaviors} that deviate from typical human reasoning patterns. Humans generally adapt their thinking based on problem complexity~\citep{HumanProblemSolve1972}. In contrast, as illustrated in Fig.~\ref{fig:intro-fig}, DeepSeek-R1~\citep{DeepSeekR125} tends to generate longer reasoning on a simpler sub-problem than on its complex composite question (1197 vs. 817 tokens), while achieving lower accuracy (38\% vs. 50\%)\footnote{For each problem, we generated multiple samples to account for randomness.}. 

We also identify this unexpected phenomenon across a wide range of reasoning models, as shown in Fig.~\ref{fig:compo_quali}. This is primarily because prevalent training paradigms generally overlook the high variability of Chain-of-Thought (CoT)~\citep{CoT22} data during the training phase. These CoT data are heuristically curated by human annotators or generated through online rollout~\citep{PPO17,DeepseekMath24}, rarely constrained by explicit principles, \textit{e.g.}, how much thinking budget to allocate for a given problem~\citep{CoTLength25}. Hence, such approaches fail to guide models toward an optimal thinking strategy. It will lead to \textit{inefficient allocation of computation}—either overthinking~\citep{OverThinkingO1Like24,StopOverthinking25} or underthinking~\citep{BetweenUnderThinkingandOverthinking25,UnderthinkingO1Like25}, which in turn harms the performance~\citep{ChainofThoughtless24, HiddenRiskReason25}.

\begin{figure*}[!t]
    \centering
    \includegraphics[width=1.0\textwidth]{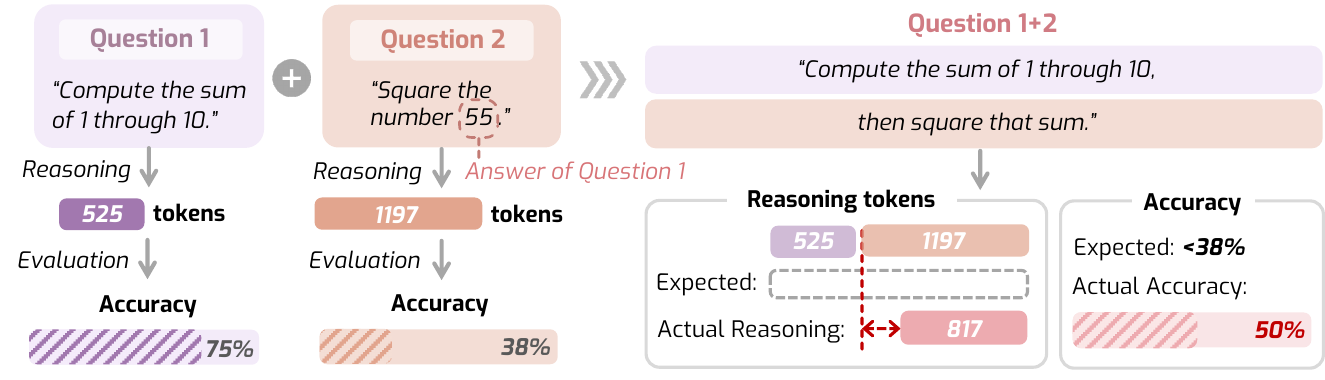}
    \caption{\textbf{Illustrative example with DeepSeek-R1} on (a) a summation question, (b) a squaring question, and (c) their composition (``sum, then square''). The model allocates \textasciitilde300 \textit{more} reasoning tokens to solve the squaring question than to the composite question, with a 12.5\% \textit{accuracy drop}. }
    \label{fig:intro-fig}
    \vspace{-0.5cm}
\end{figure*}

To overcome this limitation, one line of work focuses on adaptive \textit{post-training} techniques, including supervised fine-tuning with variable-length CoT~\citep{L125, Kimi15}. Another line of work modulates reasoning at \textit{test time}~\citep{S1_25,CyclicReflex25,AlphaOne25}. While many attempts have been made to control reasoning, existing methods lack a unified theoretical framework for regulating optimal reasoning and still behave undesirably in our studies. Therefore, beyond empirical methods, several key challenges remain: \textbf{(Q1)} \textit{Can we theoretically formalize model reasoning to ensure desirable behavior?} (Section \ref{sec:theory}) \textbf{(Q2)} \textit{How can we evaluate whether popular LRMs follow these proposed principles?} (Section \ref{4-benchmark}) \textbf{(Q3)} \textit{Does enforcing these principles further improve general reasoning capabilities?} (Section \ref{sec:align}, \ref{sec:results})

To fill this gap, we introduce \textbf{the Laws of Reasoning} (\ours), which systematically formalize the relationship
between complexity and optimal reasoning behaviors in \textit{ideal} LRMs. The \ours\space framework comprises a core \textit{compute law} and a complementary \textit{accuracy law}. Given the practical challenges of measuring these hypotheses, the two fundamental laws are approximated via two tractable properties of optimal reasoning models, monotonicity and compositionality. 

We then evaluate whether current LRMs follow the laws by developing \oursbench, \textbf{a comprehensive benchmark} that examines monotonicity and compositionality in LRMs. While \oursmonobench\space is a synthetic benchmark with curated complexity orderings for evaluating monotonicity, \ourscompobench\space is constructed from MATH500~\citep{MATH500_24} to measure compositionality. Our results show that current models exhibit monotonicity but \emph{lack compositionality}, even for strong baselines. 

In response, we propose \textbf{a simple yet effective fine-tuning approach} to enforce the compute-law compositionality. From validation experiments, we present three key insights: (1) the compositionality of reasoning compute can be greatly improved with \emph{simple fine-tuning} approach;
(2) Enforcing compositionality \textit{generally leads to better reasoning capability};
(3) \textit{Synergistic effects} emerge, yielding broader improvements across different properties and laws. 
\vspace{-0.15cm}

\section{The Laws of Reasoning}
\label{sec:theory}

\vspace{-0.1cm}

We introduce the Laws of Reasoning (\ours), a unified framework that formalizes the relationship between question complexity and optimal reasoning behaviors for \textit{ideal} models. Specifically, we focus on two key aspects, reasoning compute and accuracy, which are fundamental to understanding how models scale, generalize, and allocate computation when solving complex problems. Section~\ref{law:notation} formulates the key concepts of reasoning. In Section~\ref{law:compute_law}, we present the central \textit{compute law}, with a hypothesis that the reasoning budget should scale proportionally with complexity. In Section~\ref{law:accuracy_law}, we introduce the complementary \textit{accuracy law}, which posits that overall accuracy should decay exponentially with increasing complexity. See Fig.~\ref{fig:law-fig} for an illustration of the framework.

\vspace{-0.1cm}

\subsection{Problem Formulation}\label{law:notation}
\vspace{-0.05cm}
\paragraph{Notation.} 
Let $x\in\mathcal{X}\subseteq\mathcal{V}^{\ast}$ denote a question, where $\mathcal{V}^\ast$ is the space of finite-length sequences over a vocabulary $\mathcal{V}$. Let $M_\theta\in\mathcal{M}$ denote an \emph{autoregressive} large reasoning model. LRMs adopt the \textit{thinking-then-answering} paradigm~\citep{DeepSeekR125, Phi4-Reasoning, Gemini25}, where the model $M_\theta$ first generates a reasoning chain $r\in\mathcal{R}\subseteq\mathcal{V}^{\ast}$ with probability $p_{\theta}(r\mid x)$ and then an answer $y\in\mathcal{Y}\subseteq\mathcal{V}^{\ast}$ with probability $p_{\theta}(y\mid x,r)$. We assume a fixed decoding strategy by default and denote the model’s output by $o\!=\!(r,y)\in\mathcal{O}\subseteq\mathcal{V}^{\ast}$. We define the composition of two questions $x_1$ and $x_2$ as their concatenation with a connector prompt $c \in \mathcal{V}^\ast$\footnote{One example of $c$ can be \texttt{“Answer the following questions in order: Q1. \{Q1\}$\backslash$nQ2. \{Q2\}”}.}, \ie{}, $x_1 \oplus x_2 = \mathrm{concat}(x_1, c, x_2)$. 

\begin{figure*}[!t]
    \centering
    \includegraphics[width=1\textwidth]{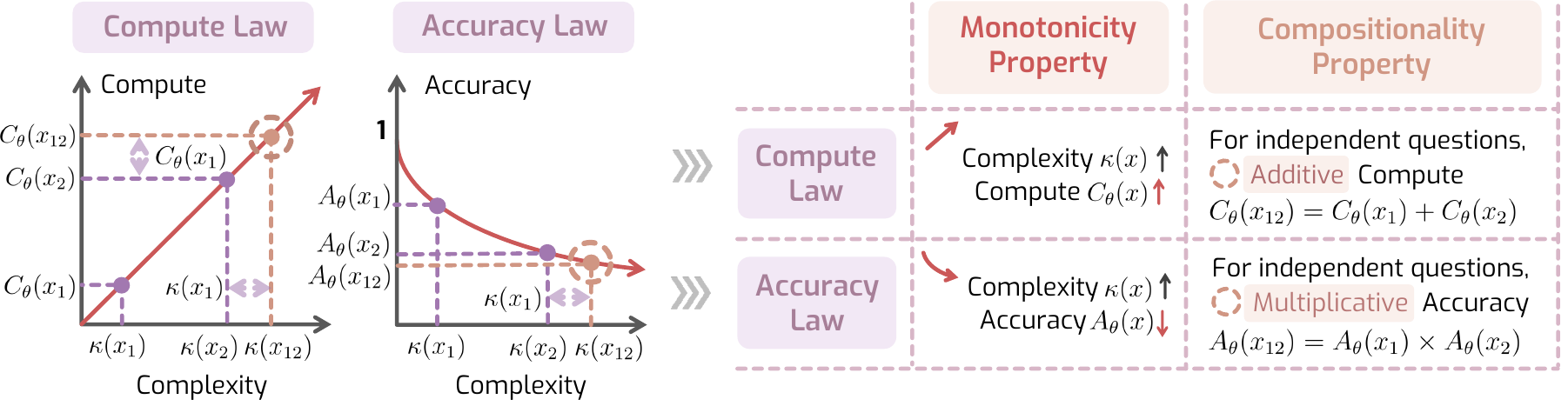}
    \caption{\textbf{Overview of the \textsc{LoRe} Framework.} We present the compute law with the complementary accuracy law: compute scales linearly \raisebox{-0.25em}{\includegraphics[height=0.9em]{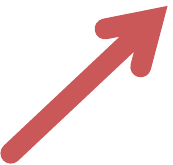}} and accuracy decays exponentially \raisebox{-0.25em}{\includegraphics[height=0.9em]{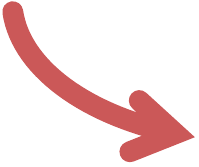}} with question complexity. Our framework approximates these laws using two properties: \textit{monotonicity} and \textit{compositionality}. Specifically, for the compute law, monotonicity captures that more complex questions require more compute, while compositionality \raisebox{-0.2em}{\includegraphics[height=0.9em]{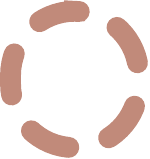}} implies additive compute for independent questions.}
    \label{fig:law-fig}
    \vspace{-0.5cm}
\end{figure*}

\begin{definition}[Complexity]\label{def:complexity} 
    Let a \emph{unit-cost primitive step} denote a single valid transition of a fixed deterministic Turing machine~\citep{Turing37}, and let $\tau$ be any finite sequence of primitive steps with length $\ell(\tau)\in\mathbb{N}$.
    Let $v(x,\tau)\in\{0,1\}$ be a binary verifier that accepts $(x,\tau)$ if and only if $\tau$ is a valid solution sequence for $x$.
    The \emph{complexity} of $x\in\mathcal{X}$ is
    \begin{equation*}
    \kappa(x)\;\triangleq\;\min\{\,\ell(\tau):\ v(x,\tau)=1\,\}\ \in\ \mathbb{N}\cup\{\infty\},
    \end{equation*}
    with $\kappa(x)=\infty$ if no valid solution sequence exists.
\end{definition}

Unlike difficulty measured by empirical failure rates
~\citep{yu2026dapo,wen2025light,sun2026improving}
or estimated using psychometric models ~\citep{truong2025reliable,sun2026your,zhou2026lost}, our notion of complexity refers to the \textit{minimal number of unit-cost primitive steps}, \textit{i.e.}, the length of the shortest correct solution. Conceptually, $\kappa(x)$ can be well-defined via a binary verifier and a fixed deterministic Turing machine. However, computing $\kappa(x)$ is generally intractable, as verifying the \textit{minimal} solution length requires a global search over a potentially exponential space.

For a given model, its test-time reasoning compute is directly proportional to the number of reasoning tokens generated. We therefore quantify reasoning compute as follows.

\begin{definition}[Reasoning Compute]\label{def:compute} The \emph{reasoning compute} on question $x$ is defined as the expected number of reasoning tokens generated by the model:
\begin{align*}
    C_\theta(x) \triangleq \mathbb{E}_{r \sim p_\theta(\cdot \mid x)}[\ell(r)],
\end{align*}
where $\ell(r)$ denotes the length (in tokens) of the reasoning chain $r$.
\end{definition}

\begin{definition}[Reasoning Accuracy]\label{def:accuracy}
The \emph{reasoning accuracy} is defined as the probability that the model, when generating a reasoning chain and an answer given input $x$, produces a final answer that matches the ground truth. Formally,
\begin{align*}
A_{\theta}(x) \triangleq \mathbb{E}_{(r,y)\sim p_\theta(\cdot \mid x)}\left[\mathbf{1}\left\{ \operatorname{ans}(y) = a^\star(x) \right\}\right] \in [0,1],
\end{align*}
where $a^\star(x)$ denotes the correct answer to $x$, and $\operatorname{ans}(y)$ extracts the final answer from $y$.
\end{definition}

\subsection{Compute Law}\label{law:compute_law}

We hypothesize that, if a reasoning model allocates its reasoning compute efficiently, the amount of compute is expected to scale proportionally with complexity in approximation, \ie{}, $C_\theta(x) \propto \kappa(x)$:

\begin{restatable}[Compute Law]{hypothesis}{ComputeLaw}\label{def:compute_general}
    For an optimal reasoning model $M_\theta$ and a question $x$ with complexity $\kappa(x)$, there exist \(\alpha_\theta>0\) with,
    \begin{equation*}
    C_\theta(x) = \alpha_\theta\,\kappa(x) + o(\kappa(x)),
    \end{equation*}
    for some \(\alpha_\theta>0\) that depends only on $M_{\theta}$ and the decoding strategy. \(o(\kappa(x))\) denotes a small systematic overhead that is sublinear, \textit{i.e.}, $o(\kappa)/\kappa \to 0$ when $\kappa\to\infty$. 
\end{restatable}

Specifically, the \(o(\kappa(x))\) term captures the introductory and transition tokens during the reasoning process. These tokens generally constitute a very small portion of the overall reasoning and can therefore be ignored in practice. Further discussion on compute definitions is provided in Appendix~\ref{app:compute-definition}.

\textbf{Two Tractable Alternative Properties as Proxies.} As discussed in Definition~\ref{def:complexity}, the complexity $\kappa(x)$ is difficult to measure in practice. Consequently, empirically validating the linear relationship is nontrivial, as it would require known complexity values for individual questions. To address this, we adopt two tractable properties as empirical proxies for studying the laws: \textit{monotonicity} and \textit{compositionality}. These properties offer two key advantages: (i) they are tractable to verify without access to the exact value of $\kappa(x)$. Monotonicity relies only on relative comparisons between questions, while compositionality tests whether compute is additive over independent question pairs;
 (ii) they are theoretically sufficient to imply the proposed compute law (Proposition~\ref{def:compute_combine}).

\begin{property}[Compute-Complexity Monotonicity]\label{pro:compute_mono}
    For $\,x_1,x_2\in\mathcal{X}$, the reasoning compute is \textit{monotonically non-decreasing} with complexity:
    \begin{align*}
        \kappa(x_1)\le \kappa(x_2)\ \Longrightarrow\
    C_{\theta}(x_1)\le C_{\theta}(x_2).
    \end{align*}

\end{property} 
\vspace{-0.1cm}
For compositionality, when two sub-questions are related, their composition may involve shared reasoning efforts that cannot be reliably quantified. To examine compositionality in a controllable manner, we therefore require the sub-questions to be independent.

\begin{definition}[Independence]\label{def:independence} For $x_1,x_2\in\mathcal{X}$, $x_1$ and $x_2$ are \emph{independent} if the complexity of their composition is additive, \textit{i.e.}, $\kappa(x_1 \oplus x_2) = \kappa(x_1) + \kappa(x_2).$
\end{definition}
Since the exact complexity is difficult to obtain, we define independence operationally. Suppose each \(x \in \mathcal{X}\) is associated with a set of mathematical concepts\footnote{For example, concepts may come from Calculus (\textit{e.g.}, derivatives), Algebra (\textit{e.g.}, group theory), or Discrete mathematics (\textit{e.g.}, logic).} \(\mathcal{S}(x) \subseteq \mathcal{S}\) relevant to solving it. We consider two questions \(x_1\) and \(x_2\) independent if their concept sets are disjoint, \textit{i.e.}, $\mathcal{S}(x_1) \cap \mathcal{S}(x_2) = \varnothing$.

\begin{property}[Compute-Complexity Compositionality]\label{pro:compute_compo}
    For $\,x_1,x_2\in\mathcal{X}$, if \(x_1\) and \(x_2\) are independent, their composite $x_1 \oplus x_2$ exhibits \textit{additive compute}:
    \begin{align*}
        C_\theta(x_1 \oplus x_2) = C_\theta(x_1) + C_\theta(x_2) + o(\kappa(x_1) + \kappa(x_2)),
    \end{align*}
where the sublinear terms account for systematic overhead in the reasoning process (as assumed in Hypothesis~\ref{def:compute_general}). Therefore, the reasoning compute is \emph{approximately additive}:
    \begin{align*}
    C_\theta(x_1 \oplus x_2) \approx C_\theta(x_1) + C_\theta(x_2).
    \end{align*}

\end{property}
\vspace{-0.2cm}
\noindent
\textbf{Discussion.} Intuitively, these properties are motivated by two basic principles: (i) more complex questions naturally require more reasoning; (ii) two independent sub-questions involve no overlapping reasoning, so the total compute is the sum of solving each one individually.  In the next proposition, we state informally that these properties imply the compute law (Hypothesis~\ref{def:compute_general}); a formal proof is provided in Appendix~\ref{appx:linear-proof}. These tractable properties offer a practical way to evaluate whether current LRMs follow the compute law.

\begin{proposition}\label{def:compute_combine}
    Under certain conditions, if a reasoning model $M_\theta$ satisfies compute-complexity monotonicity and compositionality, then its reasoning compute $C_\theta(x) \propto \kappa(x)$ for \(x\in\mathcal{X}\).
\end{proposition}

\vspace{-0.3cm}
\subsection{Beyond Compute: Accuracy Law}\label{law:accuracy_law}

Following Definition~\ref{def:complexity}, suppose a question requires solving \(\kappa(x)\) unit-cost primitive steps. If each step succeeds independently with a fixed probability and all steps must succeed for the final answer to be correct, then the overall accuracy is expected to decay exponentially with \(\kappa(x)\). This intuition motivates the following formulation of the accuracy law:

\begin{hypothesis}[Accuracy Law]\label{def:accuracy_general}
    For an optimal reasoning model $M_\theta$ and a question $x$ with complexity $\kappa(x)$, when \(0<A_{\theta}(x)\leq1\), there exists the decay rate $\lambda_\theta \ge 0$ with,
    \begin{equation*}
    A_{\theta}(x)\;=\;\exp\!\big(-\lambda_\theta\,\kappa(x)\big).
    \end{equation*}
    Equivalently,
    \(\log A_\theta(x)\propto -\kappa(x)\).
   
\end{hypothesis}

Similar to the compute law, we assume that the reasoning accuracy for the optimal reasoning model $M_\theta$ satisfies two fundamental properties: \textit{monotonicity} and \textit{compositionality}.

\begin{property}[Accuracy-Complexity Monotonicity]\label{pro:accuracy_mono}
    For $\,x_1,x_2\in\mathcal{X}$, the reasoning accuracy is \textit{monotonically non-increasing} with complexity:
    \begin{align*}
    \kappa(x_1)\le \kappa(x_2)\ \Longrightarrow\
    A_{\theta}(x_1)\ge A_{\theta}(x_2).
    \end{align*}

\end{property}

\begin{property}[Accuracy-Complexity Compositionality]\label{pro:accuracy_compo}
    For $\,x_1,x_2\in\mathcal{X}$, if \(x_1\) and \(x_2\) are independent, their composite $x_1 \oplus x_2$ exhibits \textit{multiplicative accuracy}:
    \begin{align*}
    A_{\theta}(x_1\oplus x_2)=A_{\theta}(x_1)\cdot A_{\theta}(x_2).
    \end{align*} 

\end{property} 
\noindent
\textbf{Discussion.} These properties are motivated by two basic principles: (i) more complex questions tend to have lower accuracy; (ii) for two independent questions with accuracies \(p_1\) and \(p_2\) (e.g., \(p_1 = 0.8\), \(p_2 = 0.7\)), the probability of correctly answering both should be \(p_1 \cdot p_2\) (e.g., \(0.56\)). We state below that these properties imply the accuracy law, with a formal proof provided in Appendix~\ref{appx:linear-proof}.

\begin{restatable}{proposition}{MainProp}\label{def:accuracy_combine}
    Under certain conditions, if a reasoning model $M_\theta$ satisfies accuracy-complexity monotonicity and compositionality, then its reasoning accuracy $\log A_\theta(x)\propto -\kappa(x)$ for \(x\in\mathcal{X}\).
\end{restatable}

\vspace{-0.2cm}
\section{Do Current LRMs Follow the Laws?}\label{4-benchmark}
\vspace{-0.1cm}

In this section, we evaluate whether current LRMs follow the \textit{reasoning laws}, since they may not necessarily satisfy these laws in practice. Hence, we propose \textsc{LoRe-Bench}, a two-fold benchmark that leverages two tractable properties, monotonicity and compositionality, to examine LRMs.

\begin{figure*}[!t]
    \centering
    \includegraphics[width=\textwidth]{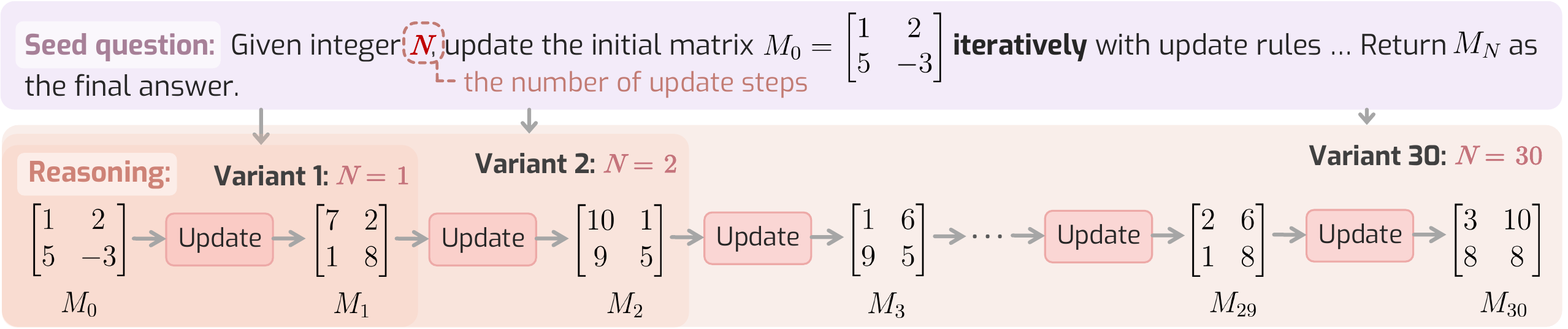}
    \caption{\textbf{Question Generation of \oursmonobench.} For each seed question, we generate 30 variants with increasing complexity. Specifically, variant $N$ applies the update rules $N$ times to derive the answer, so the question complexity increases monotonically with $N$.}
    \label{fig:lore-mono}
    \vspace{-0.55cm}
\end{figure*}

\vspace{-0.1cm}
\subsection{LoRe-Mono} \label{lore-mono}
\vspace{-0.05cm}
 
Evaluating the monotonicity property in Property~\ref{pro:compute_mono} or Property~\ref{pro:accuracy_mono} requires comparing the complexity of arbitrary question pairs. However, due to its definition via minimal solution length, complexity is inherently difficult to quantify in practice. As a result, existing benchmarks are not suited for such analysis. To address this challenge, we construct \textsc{LoRe-Mono}, a synthetic benchmark where questions are carefully curated and validated to follow known complexity orderings. \textsc{LoRe-Mono} comprises $1200$ evaluation instances derived from seed questions across four reasoning-intensive domains, allowing us to systematically assess the monotonicity of reasoning compute and accuracy.


\noindent\textbf{(1) Seed Question Curation.} We select four reasoning-intensive domains—\emph{math}, \emph{science}, \emph{language}, \emph{code}—and curate 10 diverse seed questions for each. A seed question defines a problem template shared across its variants. 

\textbf{(2) From Seed Questions to Variants.} As shown in Fig.~\ref{fig:lore-mono}, for each seed question, we create a series of variants (30 in total) that become increasingly complex by requiring more steps to reach the final answer. For example, variant 1 requires one matrix operation, and variant 30 requires thirty, with the identical operation applied repeatedly. By design, a larger number of steps corresponds to higher complexity. 

\textbf{(3) Program-based Generation and Manual Verification.}
All variants are generated through Python scripts to ensure correctness and scalability. To prevent unintended shortcuts such as periodic patterns, we manually verify each seed question and review sampled variants. We provide detailed seed questions and variants for each domain in Appendix~\ref{app:lore-bench-q-examples}.

We use the Spearman correlation coefficients $\rho \in [-1, 1]$
to measure how the variant index, which directly determines the constructed question’s complexity, relates to two quantities: reasoning compute and log accuracy. A high correlation with reasoning compute indicates that compute grows monotonically with complexity (Property~\ref{pro:compute_mono}), while a negative correlation with log accuracy indicates that accuracy tends to degrade as complexity increases (Property~\ref{pro:accuracy_mono}).

\vspace{-0.1cm}

\subsection{LoRe-Compo}\label{lore-compo}

In contrast, assessing compositionality is more straightforward: it only requires taking any two independent questions as sub-questions and constructing a composition. This benchmark serves as a controlled, quantitative testbed for rigorously evaluating compositional behaviors. We build \textsc{LoRe-Compo} from MATH500~\citep{MATH500_24}, where each question is labeled by subject (e.g., Algebra, Geometry). Specifically, we randomly sample a pair of questions $(x_1, x_2)$ from \textit{distinct subjects} to ensure independence, and concatenate them into a composite question $x_{12}$. Each original question is used only once, yielding $250$ triplets, each with two sub-questions and their composition: $\mathcal{D}_{\text{LoRe-Compo}} = \{ (x_1^{(i)}, x_2^{(i)}, x_{12}^{(i)}) \}_{i=1}^{250}$. Recall that for a function $f_\theta(\cdot)$ (either $C_\theta(\cdot)$ or $\log A_\theta(\cdot)$), compositionality implies that
$f_\theta(x_{12}) \approx f_\theta(x_1) + f_\theta(x_2)$.
We therefore quantify the degree to which a model follows this property using the mean absolute deviation ($\mathrm{MAD}$):
\[
\mathrm{MAD}_f
=
\mathbb{E}_{(x_1,x_2,x_{12})\sim \mathcal{D}_{\text{LoRe-Compo}}}
\left[
\left|
f_\theta(x_{12})
-
\bigl(f_\theta(x_1)+f_\theta(x_2)\bigr)
\right|
\right].
\]
\vspace{-0.3cm}

A smaller MAD indicates stronger adherence to the compositionality property. However, MAD is scale-dependent. To address this, we adopt the Normalized MAD ($\mathrm{nMAD}$): 
\[
\mathrm{nMAD}_f = \frac{\mathrm{MAD}_f}{S_f}, 
\quad
S_f =
\mathbb{E}_{(x_1,x_2,x_{12})\sim \mathcal{D}_{\text{LoRe-Compo}}}
\left[
\left|
f_\theta(x_1)+f_\theta(x_2)
\right|
\right].
\]
\vspace{-0.5cm}

\subsection{Findings and Analysis}\label{lore-result}
\vspace{-0.05cm}

\textbf{Evaluation Setups.}
We examine $11$ LRMs on \textsc{LoRe-Bench}: $8$ standard open-sourced models—
DeepSeek-R1-Distill (Qwen-1.5B, Qwen-7B, Llama-8B, Qwen-14B)~\citep{DeepSeekR125}, Phi-4-mini-reasoning~\citep{xu2025phi}, OpenReasoning-Nemotron-14B~\citep{OpenCodeReasoning}, Sky-T1-32B-Preview~\citep{SkyT1} and Qwen3-Next-80B-A3B-Thinking~\citep{Qwen3Report}—and $2$ models that apply reasoning length control, Thinkless-1.5B-RL-DeepScaleR~\citep{fang2025thinkless} and AdaptThink-7B-delta0.05~\citep{zhang2025adaptthink}, and $1$ proprietary model, Gemini 3.1 Flash-Lite. For each question, we sample 8 outputs per model with a fixed temperature (0.8 for the Phi-4 family and 0.6 for other model families from technical reports) and a maximum length of 20480 tokens. For \textsc{LoRe-Mono}, at each variant index we average reasoning compute\footnote{We apply max-min normalization to the reasoning compute of each question to prevent any items from dominating the results.} and log accuracy over $40$ questions then compute the Spearman correlation.

\begin{table*}[ht!]
\vspace{-0.1cm}
\caption{\textbf{Monotonicity Results on \textsc{LoRe-Mono}.} We examine whether LRMs satisfy monotonicity. Spearman correlations are reported for reasoning compute and log accuracy. \textit{Lang.} denotes \textit{Language}.}
\label{table:mono-before-sft}
\centering
\renewcommand{\arraystretch}{1.05}
\setlength{\tabcolsep}{1.7mm}
\resizebox{1\linewidth}{!}{
\begin{tabular}{cccccccccccc}
\toprule[0.95pt]
\multirow{2}{*}[-0.7ex]{\textbf{Size}} & \multirow{2}{*}[-0.7ex]{\textbf{Models}} & \multicolumn{5}{c}{\textbf{Reasoning Compute $\uparrow$}} & \multicolumn{5}{c}{\textbf{Log Accuracy $\downarrow$}} \\
\cmidrule(lr){3-7}\cmidrule(lr){8-12}
& & \textbf{Math} & \textbf{Science} & \textbf{Lang.} & \textbf{Code} & \textbf{All} & \textbf{Math} & \textbf{Science} & \textbf{Lang.} & \textbf{Code} & \textbf{All}\\
\midrule[0.6pt]

\multirow{2}{*}{\textbf{1.5B}} 
& DeepSeek-R1-1.5B   & 0.861 & 0.910 & \cellcolor[RGB]{250,235,215}-0.346 & \cellcolor[RGB]{250,235,215}0.151 & \cellcolor[RGB]{250,235,215}\textbf{0.875} & -0.795 & -0.864 & \cellcolor[RGB]{250,235,215}-0.210 & \cellcolor[RGB]{250,235,215}-0.487 & \cellcolor[RGB]{250,235,215}\textbf{-0.868}\\
& Thinkless-1.5B     & 0.943 & 0.961 & 0.648  & 0.794 & \textbf{0.976} & -0.951 & -0.934 & -0.556 & -0.539 & \textbf{-0.960}\\
\midrule[0.4pt]

\textbf{3.8B} 
& Phi-4-mini & 0.980 & 0.973 & 0.936  & 0.922 & \textbf{0.988} & -0.965 & -0.802 & -0.911 & -0.822 & \textbf{-0.954}\\
\midrule[0.4pt]

\multirow{2}{*}{\textbf{7B}}
& DeepSeek-R1-7B      & 0.956 & 0.975 & 0.901  & 0.970 & \textbf{0.991} & -0.946 & -0.876 & -0.899 & -0.818 & \textbf{-0.978}\\
& AdaptThink-7B       & 0.984 & 0.995 & 0.950  & 0.984 & \textbf{0.995} & -0.963 & -0.949 & -0.904 & -0.888 & \textbf{-0.972}\\
\midrule[0.4pt]

\textbf{8B}
& DeepSeek-R1-8B      & 0.982 & 0.962 & 0.864  & 0.963 & \textbf{0.988} & -0.944 & -0.796 & -0.924 & -0.843 & \textbf{-0.947}\\

\midrule[0.4pt]
\multirow{2}{*}{\textbf{14B}}
& {Nemontron-14B}   & 0.964 & 0.976 & 0.917 & 0.970 & \textbf{0.993} & -0.778 & -0.751 & -0.793 & -0.911 & \textbf{-0.913} \\
& {DeepSeek-R1-14B} & 0.978 & 0.973 & 0.903 & 0.980 & \textbf{0.990} & -0.888 & -0.888 & -0.803 & -0.933 & \textbf{-0.981} \\
\midrule[0.4pt]
{\textbf{32B}} & {Sky-T1-32B}      & 0.988 & 0.982 & 0.711 & 0.874 & \textbf{0.967} & -0.939 & -0.876 & -0.783 & -0.860 & \textbf{-0.963} \\
\midrule[0.4pt]
{\textbf{80B}} & {Qwen3-80B-Next}  & 0.977 & 0.984 & 0.794 & 0.993 & \textbf{0.992} & -0.403 & -0.551 & -0.774 & -0.907 & \textbf{-0.973} \\
\midrule[0.4pt]
\textemdash{} & Gemini 3.1 Flash-Lite & 0.976 & 0.916 & 0.969  & 0.991 & \textbf{0.976} & -0.922 & -0.690 & -0.321 & -0.874 & \textbf{-0.935}\\

\bottomrule[0.95pt]
\end{tabular}
}
\vspace{-0.1cm}
\end{table*}

\begin{figure*}[!t]
    \centering
    \includegraphics[width=\textwidth]{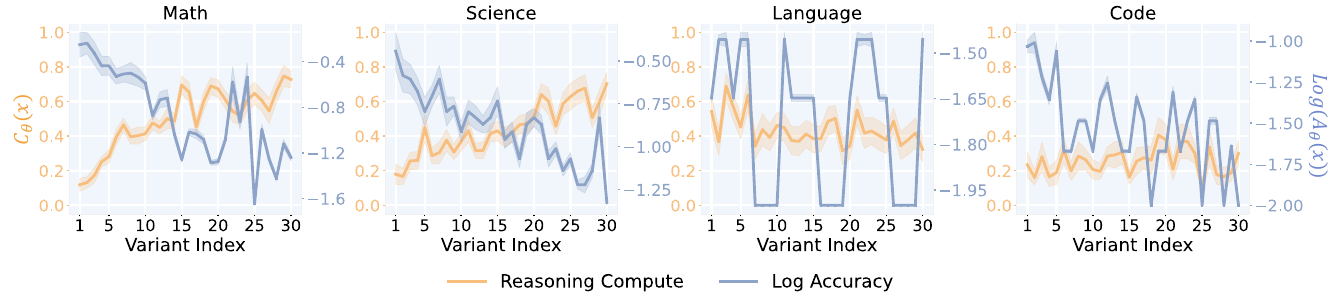}
    \vspace{-0.5cm}
    \caption{\textbf{Visualizations of Monotonicity Results on DeepSeek-R1-1.5B.} For each domain, we plot reasoning compute and log accuracy over variant index. The curves report the mean accuracy across 10 question series, and the shaded regions denote the standard deviation.}
    \label{fig:mono_quali}
    \vspace{-0.2cm}
\end{figure*}

\begin{figure}[h]
  \centering
  \begin{minipage}[t]{0.38\linewidth}
    \centering
    \vspace{0pt}
    \renewcommand{\arraystretch}{1.0}
    \parbox{6.4cm}{
    \captionof{table}{\textbf{Compositionality Results on \textsc{LoRe-Compo}.} We calculate $\mathrm{nMAD}$ for reasoning compute ($C_\theta$) and log accuracy ($\log A_\theta$).}
    \vspace{-0.1cm}
    \label{tab:compo_results}}
    \resizebox{1.2\linewidth}{!}{
    \begin{tabular}{@{}ccc@{}}
      \toprule
      \textbf{Models} & $\text{nMAD}_{C_\theta} \downarrow$ & $\text{nMAD}_{\log A_\theta} \downarrow$ \\
      \midrule
      DeepSeek-R1-1.5B   & 0.528 & 2.368 \\
      Thinkless-1.5B & 0.339 & 0.694 \\
      Phi-4-mini     & 0.322 & 0.732 \\
      DeepSeek-R1-7B     & 0.337 & 1.170 \\
      AdaptThink-7B  & 0.327 & 0.791 \\
      DeepSeek-R1-8B     & 0.423 & 0.818 \\
      {DeepSeek-R1-14B} & 0.368 & 1.310\\
      {Nemotron-14B} & 0.354 & 1.188\\
      {Sky-T1-32B} & 0.392 & 1.900 \\
      {Qwen3-80B-Next} & 0.411 & 1.487\\
      Gemini 3.1 Flash-Lite & 0.343 & 0.490 \\
      \bottomrule
    \end{tabular}
    }
  \end{minipage}%
  \hfill
  \begin{minipage}[t]{0.50\linewidth} 
    \centering
    \vspace{0pt}
    \includegraphics[width=\linewidth]{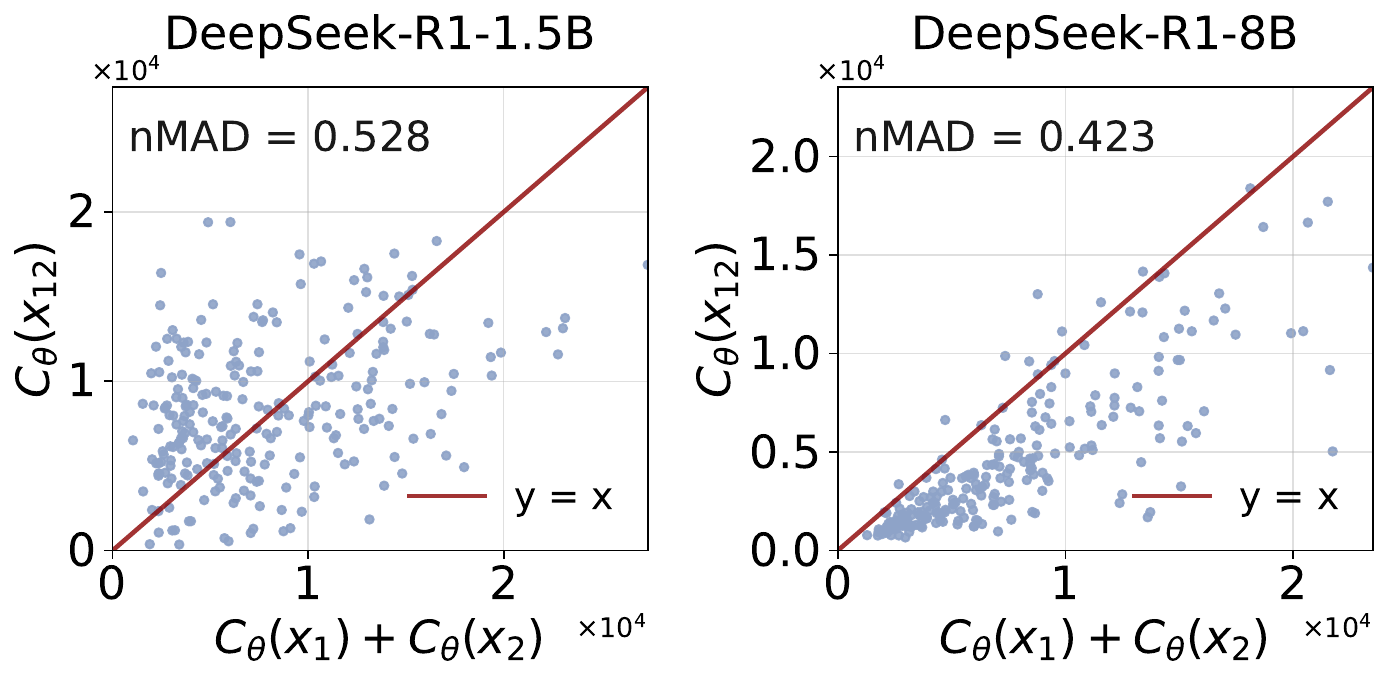}
    \caption{\textbf{Visualizations of Compositionality Results on Reasoning Compute.} We plot $C_{\theta}(x_1 \oplus x_2)$ against $C_{\theta}(x_1) + C_{\theta}(x_2)$. Further results are provided in Appendix~\ref{}.}
    \label{fig:compo_quali}
  \end{minipage}
  \vspace{-0.3cm}
\end{figure}

\textbf{Current LRMs Largely Satisfy Monotonicity.} On \textsc{LoRe-Mono}, all LRMs exhibit a strong positive correlation between reasoning compute and the variant index, which directly reflects question complexity, with most overall Spearman correlations close to $1$, as shown in Tab.~\ref{table:mono-before-sft}. The only exception is DeepSeek-R1-Distill-Qwen-1.5B, which has the weakest reasoning ability among the six models and yields a lower overall correlation ($0.875$). As illustrated in Fig.~\ref{fig:mono_quali}, notably, in the language domain its correlation between reasoning compute and complexity is negative ($-0.346$), while in the code domain, it is near zero ($0.151$). This indicates that \textit{in some domains, the reasoning compute for this model does not systematically increase with complexity, and may even decrease}. We provide a case study as additional analysis along with visualization results for other models in Appendix~\ref{app:lore-add-results}. Meanwhile, most LRMs exhibit a negative correlation between log accuracy and the variant index, as expected. For DeepSeek-R1-Distill-Qwen-1.5B, however, this trend appears noticeably weaker.

\textbf{Current LRMs Fail to Exhibit Compositionality.} 
The $\mathrm{nMAD}$ is large for both reasoning compute and log accuracy (Tab.~\ref{tab:compo_results}), indicating that current LRMs do not satisfy compositionality. This issue persists across different model scales, as larger models do not necessarily better satisfy compositionality than small-scale models. Fig.~\ref{fig:compo_quali} further plots $C_{\theta}(x_1 \oplus x_2)$ against $C_{\theta}(x_1) + C_{\theta}(x_2)$ for two representative LRMs. If an LRM adhered to the compositionality law, most points would align closely with the $y=x$ line. In practice, however, the majority of points deviate substantially. Notably, even models equipped with reasoning length control mechanisms (Thinkless-1.5B and AdaptThink-7B) exhibit considerable deviations, suggesting that such techniques do not inherently promote compositional behavior. We provide additional analyses of the proprietary model, prompt variants, and commonsense reasoning evaluation in Appendix~\ref{addi-close-source}, \ref{prompt-ablation}, and~\ref{sec:commonsense}.

\vspace{-0.2cm}
\section{Improving Reasoning via Enforcing Compositionality} 
\label{sec:align}
\vspace{-0.1cm}

In Section~\ref{4-benchmark}, we showed that while most LRMs generally satisfy \textit{monotonicity}, they often fail to satisfy \textit{compositionality}. Based on Hypothesis~\ref{def:compute_general}, this observation motivates a natural question: \textit{can enforcing compositionality lead to stronger reasoning capacity?} In response, we propose a simple yet effective supervised fine-tuning (SFT) method to promote compositional behavior in LRMs. Importantly, we focus on enforcing compositionality specifically with respect to \textit{reasoning compute}, as it provides a more direct and actionable criterion for selecting supervision examples.\footnote{Accuracy compositionality is hard to enforce directly, as it does not specify which reasoning path to take for supervision.}

\textbf{Proposed Method: \texttt{SFT-Compo}. }Specifically, let \( M_\theta \) be an LRM and \( \mathcal{D}_{\text{train}} \) a training dataset. Following the construction in Section~\ref{lore-compo}, we select question pairs \( (x_1, x_2) \in \mathcal{D}_{\text{train}} \) from distinct categories and form composite questions \( x_{12} = x_1 \oplus x_2 \). For each triplet \( (x_1, x_2, x_{12}) \), we sample \( K \) model outputs \( o = (r, y) \in \mathcal{O} \) from an LRM (either the current model \( M_{\theta} \) or a stronger teacher model), where \( r \in \mathcal{R} \) is a reasoning path and \( y \in \mathcal{Y} \) is the final answer:
\[
\{ o_1^{(k)} = (r_1^{(k)}, y_1^{(k)}) \}_{k=1}^{K} \text{ for } x_1, \quad
\{ o_2^{(k)} = (r_2^{(k)}, y_2^{(k)}) \}_{k=1}^{K} \text{ for } x_2, \quad
\{ o_{12}^{(k)} = (r_{12}^{(k)}, y_{12}^{(k)}) \}_{k=1}^{K} \text{ for } x_{12}.
\]
Since compositionality is defined over reasoning paths, among the \( K^3 \) combinations \( (o_1, o_2, o_{12}) \), we consider only those where all three reasoning paths \( r_1, r_2, r_{12} \) lead to correct answers, and select the combination that best satisfies the compositionality condition:
\begin{align}
(r_1^*, r_2^*, r_{12}^*) =
\mathop{\arg\min}_{r_1,r_2,r_{12}}
\left| \ell(r_1)+\ell(r_2)-\ell(r_{12}) \right|
\;\text{s.t.}\;
r_1,r_2,r_{12}\ \text{yields a correct final answer}.
\label{eq:compositional_gap}
\end{align}
Each triplet thus yields three supervised examples: \( (x_1, o_1^*) \), \( (x_2, o_2^*) \), and \( (x_{12}, o_{12}^*) \), where \( o_i^* = (r_i^*, y_i^*) \) with \( y_i^* \) the final answer paired with \( r_i^* \) in the sampled outputs. Aggregating across all triplets, we construct the compositional supervision dataset $\mathcal{D}_{\text{comp}} = \left\{ (x_1, o_1^*),\ (x_2, o_2^*),\ (x_{12}, o_{12}^*) \mid (x_1, x_2) \in \mathcal{D}_{\text{train}} \right\}.$
We then perform SFT on \( \mathcal{D}_{\text{comp}} \) to encourage \( M_\theta \) to internalize compositional reasoning behavior.

\vspace{-0.15cm}
\section{Experiments}
\vspace{-0.1cm}

\label{sec:results}
We now empirically evaluate \texttt{SFT-Compo}, addressing two research questions: (1) whether it effectively enforces compositionality, and (2) whether it further improves the reasoning capacity of LRMs. We also provide additional insightful findings in our analysis. 

\vspace{-0.1cm}
\subsection{Experimental Setup}
\label{sec:exp_setup}
\vspace{-0.05cm}

\textbf{Model, Dataset and SFT Recipe.} 
We evaluate four LRMs: DeepSeek-R1-Distill (Qwen-1.5B, Qwen-7B, Llama-8B)~\citep{DeepSeekR125} and Phi-4-mini-reasoning~\citep{xu2025phi}. We construct a dataset of sub-question and composite-question triplets using a subset of DeepScaler~\citep{deepscaler2025}. For each question, we use DeepSeek-R1-Distill-Qwen-14B as a stronger teacher model to sample $K=8$ model outputs. We then construct the compositionality-enforced dataset \( \mathcal{D}_{\text{comp}} \) as described in Eqn.~\ref{eq:compositional_gap}, which contains 3.9K question-output pairs. We fine-tune each LRM on \( \mathcal{D}_{\text{comp}} \) for 5 epochs with a batch size of 16. Details are provided in Appendix~\ref{appx:lore_bench_additional_results}.

\textbf{Evaluation.} 
To evaluate compositionality, we use \textsc{LoRe-Compo}. For general reasoning capacity, we consider six benchmarks: GSM8K~\citep{cobbe2021training}, MATH500~\citep{MATH500_24}, AIME 2024, AIME 2025~\citep{AIME25}, AMC 2023~\citep{AMC23}, and OlympiadBench~\citep{he2024olympiadbench}. We set the maximum generation length to 10240 tokens.

\vspace{-0.05cm}
\subsection{Main Results}

\begin{figure}[H]
\vspace{-0.3cm}
    \centering
    \begin{subfigure}[b]{0.40\linewidth}
        \centering
        \includegraphics[width=0.9\linewidth]{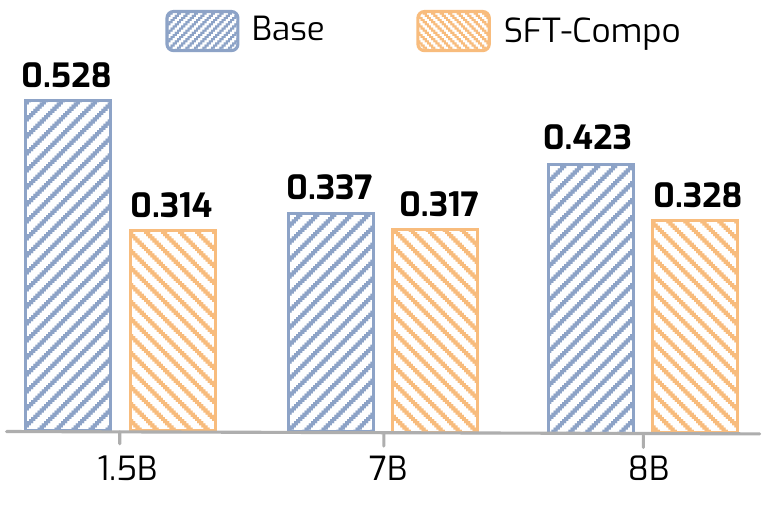}
        \vspace{-0.1cm}
        \caption{$\mathrm{nMAD}_{C_\theta}$ on \textsc{LoRe-Compo}.}
        \label{fig:compo_sft_1}
    \end{subfigure}
    \begin{subfigure}[b]{0.56\linewidth}
        \centering
        \includegraphics[width=0.9\linewidth]{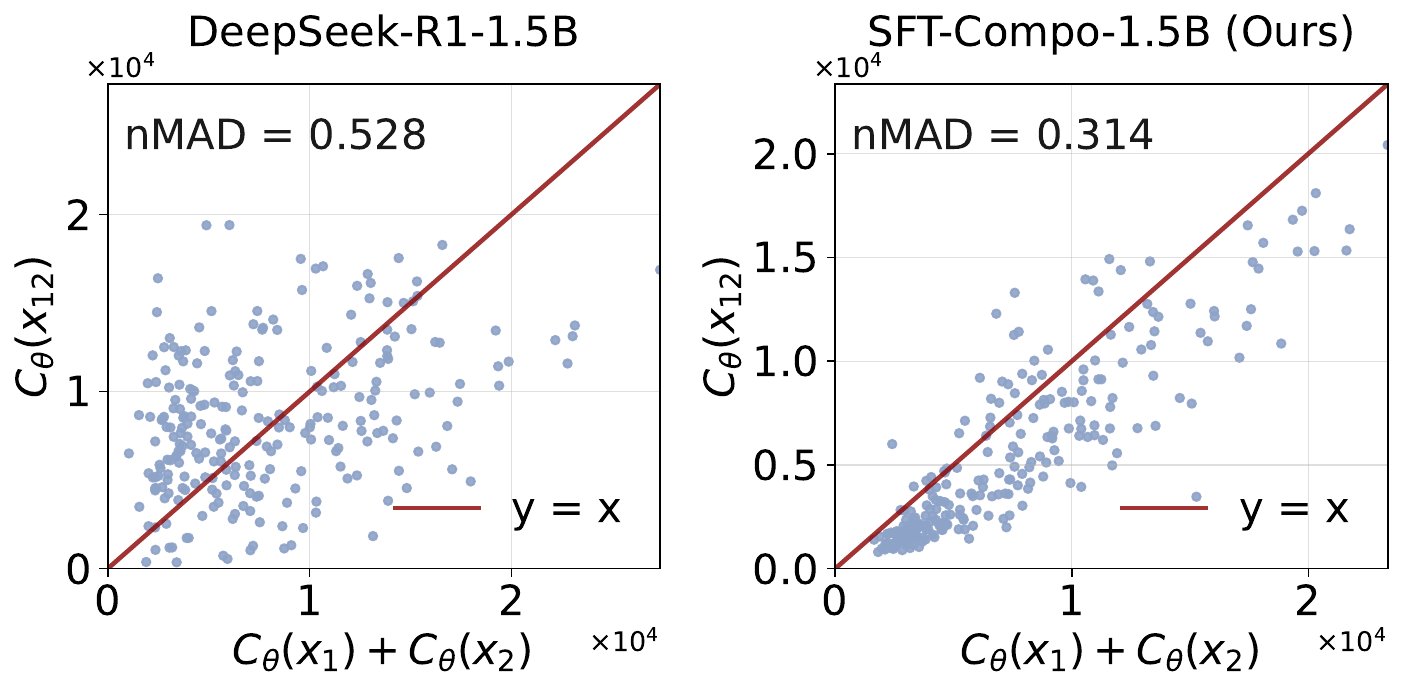}
        \vspace{-0.1cm}
        \caption{Visualizations of Reasoning Compute Compositionality.}
        \label{fig:compo_sft_2}
    \end{subfigure}
    \vspace{-0.1cm}
    \caption{\textbf{Comparison of Reasoning Compute Compositionality} on \textsc{LoRe-Compo}.}
    \label{fig:compo_sft}
    \vspace{-0.4cm}
\end{figure}



\textbf{Does \texttt{SFT-Compo} Effectively Enforce Compositionality Compared to the Base Model?}
We compare LRMs before and after SFT using the $\mathrm{nMAD}$ of reasoning compute on \textsc{LoRe-Compo}. As shown in Fig.~\ref{fig:compo_sft_1}, \texttt{SFT-Compo} consistently reduces $\mathrm{nMAD}$ compared to the base model. On the 1.5B model, \texttt{SFT-Compo} achieves a reduction from $0.528$ to $0.314$ (a $40.5$\% reduction), and on the 8B model, from $0.423$ to $0.328$ (a $22.5$\% reduction). We further visualize the results on the 1.5B model in Fig.~\ref{fig:compo_sft_2}, where \texttt{SFT-Compo} aligns much more closely with the $y{=}x$ line. Therefore, \textit{the compositionality of reasoning compute can be effectively enforced in a simple manner via \texttt{SFT-Compo}}.

\begin{table*}[!t]
\caption{\textbf{General Reasoning Evaluation Results}. We evaluate \texttt{Base} (pre-SFT), \texttt{SFT}, and \texttt{SFT-Compo} (Ours) on mathematical and science reasoning benchmarks. All numbers report Pass@1 accuracy (\%) computed over 8 sampled outputs.  $\overline{\text{Pass@1}}$ denotes the average across the six benchmarks. Numbers in \textbf{\textcolor{runpei-orange}{orange}} indicate improvements relative to the base model.}
\label{table:perfm-after-sft}
\centering
\renewcommand{\arraystretch}{1.05}
\setlength{\tabcolsep}{1.6mm}
\resizebox{1\textwidth}{!}{
\begin{tabular}{ccccccccc}
\toprule[0.95pt]
\multirow{2}{*}[-0.7ex]{\textbf{Base Model}} & \multirow{2}{*}[-0.7ex]{\textbf{Method}} & \multicolumn{5}{c}{\scshape\textbf{Math}} & \scshape\textbf{Science}  & \multirow{2}{*}[-0.7ex]{$\overline{\textbf{Pass@1}}$}\\
\cmidrule(lr){3-7}\cmidrule(lr){8-8}
&& \textbf{AIME24} & \textbf{AIME25} & \textbf{AMC23} & \textbf{MATH500} & \textbf{GSM8K} & \textbf{Olympiad} \\
\midrule[0.6pt]

\multirow{3}{*}{\textbf{DeepSeek-R1-1.5B}} 
& \texttt{Base} & 18.8 & 20.4 & 59.7 & 71.6 & 81.2 & 33.8 & 47.6\\
& \texttt{SFT} & 20.4$_{\text{+1.6}}$ & 21.5$_{\text{+1.1}}$ & 59.6$_{\text{-0.1}}$ & 76.4$_{\text{+4.8}}$ & 81.7$_{\text{+0.5}}$ & 36.1$_{\text{+2.3}}$ & 49.3$_{\text{+1.7}}$\\
& \texttt{SFT-Compo} (Ours) & \cellcolor[RGB]{250,235,215}\textbf{26.2}$_{\text{\textcolor{runpei-orange}{\textbf{+7.4}}}}$ & \cellcolor[RGB]{250,235,215}\textbf{21.7}$_{\text{\textcolor{runpei-orange}{\textbf{+1.3}}}}$ & \cellcolor[RGB]{250,235,215}\textbf{65.0}$_{\text{\textcolor{runpei-orange}{\textbf{+5.3}}}}$ & \cellcolor[RGB]{250,235,215}\textbf{77.6}$_{\text{\textcolor{runpei-orange}{\textbf{+6.0}}}}$ & \cellcolor[RGB]{250,235,215}\textbf{85.1}$_{\text{\textcolor{runpei-orange}{\textbf{+3.9}}}}$ & \cellcolor[RGB]{250,235,215}\textbf{38.7}$_{\text{\textcolor{runpei-orange}{\textbf{+4.9}}}}$ & \cellcolor[RGB]{250,235,215}\textbf{52.4}$_{\text{\textcolor{runpei-orange}{\textbf{+4.8}}}}$\\

\midrule[0.6pt]
\multirow{3}{*}{\textbf{DeepSeek-R1-7B}} 
& \texttt{Base} & 36.3 & 27.5 & 79.0 & 86.8 & 91.0 & 48.1 & 61.5\\
& \texttt{SFT} & 40.0$_{\text{+3.7}}$ & 32.5$_{\text{+5.0}}$ & 80.4$_{\text{+1.4}}$ & 88.0$_{\text{+1.2}}$ & \textbf{91.6}$_{\text{+0.6}}$ & 48.4$_{\text{+0.3}}$ & 63.5$_{\text{+2.0}}$\\
& \texttt{SFT-Compo} (Ours) & \cellcolor[RGB]{250,235,215}\textbf{43.3}$_{\text{\textcolor{runpei-orange}{\textbf{+7.0}}}}$ & \cellcolor[RGB]{250,235,215}\textbf{33.2}$_{\text{\textcolor{runpei-orange}{\textbf{+5.7}}}}$ & \cellcolor[RGB]{250,235,215}\textbf{80.6}$_{\text{\textcolor{runpei-orange}{\textbf{+1.6}}}}$ & \cellcolor[RGB]{250,235,215}\textbf{88.8}$_{\text{\textcolor{runpei-orange}{\textbf{+2.0}}}}$ & \cellcolor[RGB]{250,235,215}\textbf{91.6}$_{\text{\textcolor{runpei-orange}{\textbf{+0.6}}}}$ & \cellcolor[RGB]{250,235,215}\textbf{50.5}$_{\text{\textcolor{runpei-orange}{\textbf{+2.4}}}}$ & \cellcolor[RGB]{250,235,215}\textbf{64.7}$_{\text{\textcolor{runpei-orange}{\textbf{+3.2}}}}$\\
\midrule[0.6pt]
\multirow{3}{*}{\textbf{DeepSeek-R1-8B}}
& \texttt{Base} & 28.3 & 22.9 & 71.9 & 76.4 & 86.5 & 40.9 & 54.5\\
& \texttt{SFT} & 30.4$_{\text{+2.1}}$ & 24.2$_{\text{+1.3}}$ & 75.2$_{\text{+3.3}}$ & 82.6$_{\text{+6.2}}$ & 88.0$_{\text{+1.5}}$ & 44.7$_{\text{+3.8}}$ & 57.5$_{\text{+3.0}}$\\
& \texttt{SFT-Compo} (Ours) & \cellcolor[RGB]{250,235,215}\textbf{31.3}$_{\text{\textcolor{runpei-orange}{\textbf{+3.0}}}}$ & \cellcolor[RGB]{250,235,215}\textbf{29.2}$_{\text{\textcolor{runpei-orange}{\textbf{+6.3}}}}$ & \cellcolor[RGB]{250,235,215}\textbf{76.9}$_{\text{\textcolor{runpei-orange}{\textbf{+5.0}}}}$ & \cellcolor[RGB]{250,235,215}\textbf{83.0}$_{\text{\textcolor{runpei-orange}{\textbf{+6.6}}}}$ & \cellcolor[RGB]{250,235,215}\textbf{89.5}$_{\text{\textcolor{runpei-orange}{\textbf{+3.0}}}}$ & \cellcolor[RGB]{250,235,215}\textbf{46.8}$_{\text{\textcolor{runpei-orange}{\textbf{+5.9}}}}$ & \cellcolor[RGB]{250,235,215}\textbf{59.5}$_{\text{\textcolor{runpei-orange}{\textbf{+5.0}}}}$\\
\midrule[0.6pt]
\multirow{3}{*}{\textbf{Phi-4-mini}}
& \texttt{Base} & 32.5 & 24.6 & 71.9 & 86.4 & 91.4 & 50.0 & 59.5\\
& \texttt{SFT} & 40.0$_{\text{+7.5}}$ & 26.2$_{\text{+1.6}}$ & 73.4$_{\text{+1.5}}$ & \textbf{87.0}$_{\text{+0.6}}$ & 91.6$_{\text{+0.2}}$ & 49.5$_{\text{-0.5}}$ & 61.3$_{\text{+1.8}}$\\
& \texttt{SFT-Compo} (Ours) & \cellcolor[RGB]{250,235,215}\textbf{43.7}$_{\text{\textcolor{runpei-orange}{\textbf{+11.2}}}}$ & \cellcolor[RGB]{250,235,215}\textbf{31.7}$_{\text{\textcolor{runpei-orange}{\textbf{+7.1}}}}$ & \cellcolor[RGB]{250,235,215}\textbf{76.9}$_{\text{\textcolor{runpei-orange}{\textbf{+5.0}}}}$ & \cellcolor[RGB]{250,235,215}\textbf{87.0}$_{\text{\textcolor{runpei-orange}{\textbf{+0.6}}}}$ & \cellcolor[RGB]{250,235,215}\textbf{92.8}$_{\text{\textcolor{runpei-orange}{\textbf{+1.4}}}}$ & \cellcolor[RGB]{250,235,215}\textbf{51.2}$_{\text{\textcolor{runpei-orange}{\textbf{+1.2}}}}$ & \cellcolor[RGB]{250,235,215}\textbf{63.9}$_{\text{\textcolor{runpei-orange}{\textbf{+4.4}}}}$\\
\bottomrule[0.95pt]
\end{tabular}
}
\vspace{-0.45cm}
\end{table*}

\textbf{Does Enforcing Compositionality Lead to Stronger Reasoning Capabilities?} As shown in Tab.~\ref{table:perfm-after-sft}, \texttt{SFT-Compo} consistently improves performance across all six benchmarks and all four model sizes. For instance, on the 8B model, it yields a notable gain of $+5.0$ in average Pass@1. To rule out the possibility that performance gains stem solely from leveraging outputs generated by a stronger teacher model, we introduce a control baseline, \texttt{SFT}, which constructs the training dataset by uniformly sampling one correct reasoning path for each question in the triplet:
\begin{equation*}
(r_1^*, r_2^*, r_{12}^*) \sim 
\text{Unif}\left(\{ (r_1, r_2, r_{12}) \mid r_1, r_2, r_{12} \text{ each yield a correct final answer} \}\right).
\end{equation*}

Notably, \texttt{SFT-Compo} outperforms \texttt{SFT} in all cases, showing that the gains are not just from distilling a stronger model but from better compliance with compositionality. This supports our Hypothesis~1—that \textit{stronger models better follow reasoning laws}—and demonstrate that \textit{encouraging compositionality further enhances the reasoning capabilities of LRMs.}
\vspace{-0.2cm}

\subsection{Synergistic Effect Analysis}



\begin{figure}[H]
    \centering
    \vspace{-0.15cm}
    \begin{subfigure}[t]{0.56\linewidth}
        \centering
        \includegraphics[width=\linewidth]{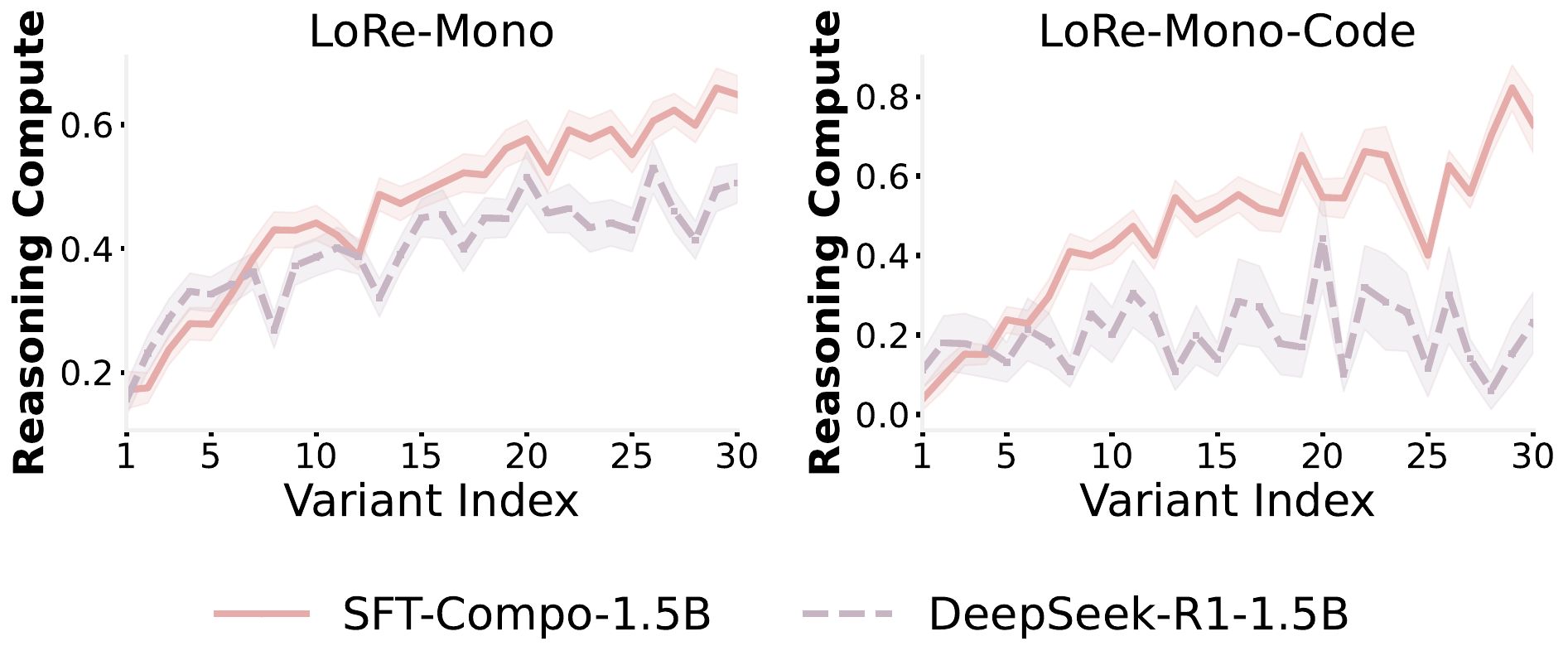}
        \vspace{-0.45cm}
        \caption{Visualizations of monoticity results on \textsc{LoRe-Mono}.}
        \label{fig:syn_1}
    \end{subfigure}
    \begin{subfigure}[t]{0.40\linewidth}
        \centering
        \includegraphics[width=0.9\linewidth]{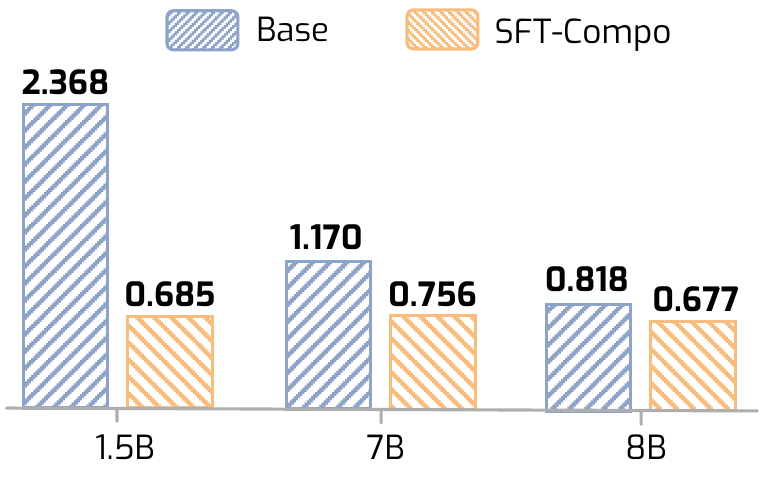}
        \vspace{-0.15cm}
        \caption{$\mathrm{nMAD}_{\log A_\theta}$ on \textsc{LoRe-Compo}.}
        \label{fig:syn_2}
    \end{subfigure}
    \vspace{-0.1cm}
    \caption{\textbf{Synergistic Effects} Among Different Reasoning Properties and Laws.}
    \label{fig:syn}
    \vspace{-0.45cm}
\end{figure}

\textbf{Enforcing Compositionality in Reasoning Compute Improves Its Monotonicity.}
Recall from Section~\ref{lore-result} that DeepSeek-R1-1.5B exhibits weak monotonicity in reasoning compute. Fig.~\ref{fig:syn_1} shows \texttt{SFT-Compo} significantly improves this property, increasing the overall Spearman correlation from $0.875$ to $0.977$. Specifically, in the code domain, it rises from $0.151$ to $0.914$. This indicates that enforcing compositionality can implicitly enhance monotonicity.

\textbf{Enforcing Compositionality in Reasoning Compute Improves Compositionality in Accuracy.}
Interestingly, though \texttt{SFT-Compo} is designed to enhance compositionality in reasoning compute, it improves the compositionality of log accuracy. Fig.~\ref{fig:syn_2} shows that the $\mathrm{nMAD}$ of log accuracy drops from $2.368$ to $0.685$ on the 1.5B model (a $71.1$\% reduction), and from $1.170$ to $0.756$ on the 7B model (a $35.4$\% reduction). This suggests a possible interplay among different reasoning laws.

\vspace{-0.15cm}
\section{Related Work}
\vspace{-0.15cm}

LRMs have emerged as a family of foundation models~\citep{FoundationModel21}. Since the advent of OpenAI o1~\citep{OpenAIO1_24}, the ``thinking-then-answering'' paradigm has been widely adopted, with notable follow-ups such as DeepSeek-R1 and Phi-4-Reasoning~\citep{Phi4_24,DeepSeekR125,QwQ32B}. Our framework builds upon the contemporary paradigm of adaptive reasoning, wherein the model's reasoning budget is dynamically controlled either through post-training interventions~\citep{O1Pruner25,zhou2025mem1} or at test time~\citep{S1_25,ChainOfDraft25,AlphaOne25}. Specifically, one line of work explores post-training techniques that modulate when and how long a model should reason~\citep{ReasonBound24,ThinkInfo25}, while another frontier focuses on dynamically adjusting reasoning behavior during inference~\citep{ConCise25, AnswerEarly25}. Refer to Appendix~\ref{appx:related-works} for additional related work.

\vspace{-0.15cm}
\section{Conclusions}
\vspace{-0.15cm}
As a comprehensive study from theoretical hypotheses to empirical validation, we advance a theoretical perspective grounded in human reasoning for improving reasoning in LRMs. We hope \ours\space can inspire potential strategies that guide models toward their optimal paradigms of thinking.

\section*{Acknowledgements and Disclosures}

The authors thank Heng Dong, Runpei Dong and Chanakya Ekbote for their insightful feedback and valuable suggestions on this project. 
Huan Zhang and Junyu Zhang are supported in part by a Schmidt Sciences Trustworthy AI award. The authors declare no competing interests.






\bibliography{main}
\bibliographystyle{iclr}

\appendix
\clearpage

\section{Reproducibility Statement}\label{appx:reproducibility}

Our code repository is available at \url{https://github.com/ASTRAL-Group/LoRe}. In Section~\ref{lore-result} and Section~\ref{sec:exp_setup}, we provide a detailed description of the experimental setup, including dataset, models, training and evaluation procedures. Additional implementation details can be found in Appendix \ref{app-imple}.

\section{Broader Impact}\label{appx:broader-impact}
This work targets complex reasoning problems with LRMs, which we believe will lead to no ethical concerns.
However, since LRMs are modern variants of LLMs, any ethical concerns raised by LLMs can potentially exist.

\section{Additional Related Work}\label{appx:related-works}
\paragraph{Large Reasoning Models.}

Large Reasoning Models (LRMs) have emerged as a family of foundation models~\citep{FoundationModel21}. Since the advent of OpenAI o1~\citep{OpenAIO1_24}, this ``thinking-then-answering'' paradigm has been widely adopted. Notably, o1-like Reasoning Models can solve increasingly complex reasoning problems through elaborate reasoning chains~\citep{CoT22,TreeOfThoughts23,GraphOfThoughts24}. Numerous efforts replicating o1's success include DeepSeek-R1 and Phi-4-Reasoning~\citep{Phi4_24,DeepSeekR125,QwQ32B}. Despite impressive progress, the internal mechanisms and behavioral patterns of reasoning in LRMs remain underexplored. \citet{shojaee2025illusion} take a step in this direction by examining reasoning through the lens of problem complexity, though their analysis is limited to a constrained puzzle-solving setting.

\paragraph{Reasoning Length Control.}

Our framework builds upon the contemporary paradigm of adaptive reasoning, in which the reasoning budget of the model is controlled either during post-training~\citep{O1Pruner25,zhou2025mem1} or at test time~\citep{S1_25,ChainOfDraft25,AlphaOne25}. One line of work develops post-training techniques that modulate when and how long a model should reason~\citep{ReasonBound24,ThinkInfo25}. This is achieved through two primary strategies: one involves supervised fine-tuning on variable-length CoT with concise yet sufficient reasoning~\citep{L125,Kimi15}; the other utilizes RL through length penalty~\citep{zhang2025adaptthink,fang2025thinkless,Laser25}. Beyond these, another frontier involves implementing dynamic control over reasoning during inference. For example, some approaches allocate inference budget via confidence~\citep{ConCise25,AnswerEarly25}, while others employ a secondary controller to modulate~\citep{SteerBudget25}.

\section{Limitations}\label{appx:limitations}
LoRe primarily focuses on two fundamental aspects of reasoning, reasoning compute and accuracy. While compute law and accuracy law capture core reasoning dynamics, there may exist other intrinsic laws. For instance, laws related to reasoning diversity, robustness, and coherence may offer complementary perspectives and could naturally extend the LoRe framework. Although we provide experiments on proprietary models, evaluating our reasoning laws ideally requires access to the full reasoning process, which is typically not available through current commercial APIs. As a result, our evaluation on proprietary models may be less reliable, and stronger conclusions would require internal measurements from service providers. It is technically plausible for the commercial service providers to evaluate their own models based on our framework, and the insights from our paper can be utilized to build better proprietary models.


\section{Proofs and Corollaries}\label{appx:linear-proof}

We first restate Proposition \ref{def:compute_combine} and Proposition \ref{def:accuracy_combine} formally and provide a complete proof, along with corresponding corollaries.

\begingroup
\renewcommand\theproposition{1}

\begin{proposition}[Formal Version]
Fix a question space $\mathcal X$, a \textit{complexity} map $\kappa:\mathcal X\to\mathbb N\cup\{\infty\}$,
and a \textit{reasoning} compute map $C_\theta:\mathcal X\to\mathbb R_{\ge0}$.
Let $\oplus$ be a binary composition operator.
For $m\ge3$ and jointly independent $x_1, \ldots, x_m$, define $x_1\oplus\cdots\oplus x_m$ by a fixed bracketing (e.g. right-associated).
Define
\[
\mathcal X_{\mathrm{fin}}:=\{x\in\mathcal X:\kappa(x)<\infty\},\qquad
K:=\kappa(\mathcal X_{\mathrm{fin}})\subseteq\mathbb N.
\]
All assumptions below are imposed on $\mathcal X_{\mathrm{fin}}$.

\begin{enumerate}
  \item[\textbf{(A1)}] \textbf{Monotonicity.} If $\kappa(x)\le \kappa(y)$ then $C_\theta(x)\le C_\theta(y)$.
  \item[\textbf{(A2)}] \textbf{Additivity under composition of independent questions.}
  If $x,y$ are independent, then
  \[
  \kappa(x\oplus y)=\kappa(x)+\kappa(y),\qquad
  C_\theta(x\oplus y)=C_\theta(x)+C_\theta(y).
  \]
  \item[\textbf{(A3)}]\label{ass:pair-independent}
  For every $u\in K$ and every $m\in\mathbb N$, there exist
  $x_1,\dots,x_m\in\mathcal X_{\mathrm{fin}}$ such that
  \[
  \kappa(x_i)=u\ \text{ for all }i,\qquad \{x_1,\dots,x_m\}\ \text{is jointly independent.}
  \]
  Consequently, $x_1\oplus\cdots\oplus x_m$ is valid and
  \[
  \kappa(x_1\oplus\cdots\oplus x_m)=mu\in K,\qquad
  C_\theta(x_1\oplus\cdots\oplus x_m)=\sum_{i=1}^m C_\theta(x_i).
  \]
\end{enumerate}

Then there exists a constant $\alpha_\theta\ge 0$ such that
\[
C_\theta(x)=\alpha_\theta\,\kappa(x)\qquad\text{for all }x\in \mathcal{X}_{\mathrm{fin}}.
\]
\end{proposition}



\begin{proof}
Define an equivalence relation $x\sim y \iff \kappa(x)=\kappa(y)$.
By (A1), $\kappa(x)\le\kappa(y)$ and $\kappa(y)\le\kappa(x)$ imply $C_\theta(x)\le C_\theta(y)$ and $C_\theta(y)\le C_\theta(x)$, hence
$C_\theta(x)=C_\theta(y)$ whenever $x\sim y$.
Thus there is a well-defined $f:K\to\mathbb R_{\ge0}$ with
$f(n)=C_\theta(x)$ for any $x$ such that $\kappa(x)=n$.

Fix $u\in K$ and $m\in\mathbb N$.
By (A3) choose jointly independent $x_1,\dots,x_m$ with $\kappa(x_i)=u$.
By (A2) and the fixed bracketing,
\[
\kappa(x_1\oplus\cdots\oplus x_m)=mu,\qquad
C_\theta(x_1\oplus\cdots\oplus x_m)=\sum_{i=1}^m C_\theta(x_i)=m\,f(u),
\]
so
\[
f(mu)=m\,f(u)\qquad(\forall\,u\in K,\ \forall\,m\in\mathbb N).
\]

If $K=\{0\}$ then $C_\theta\equiv0$ and the claim holds with $\alpha_\theta=0$.
Otherwise take $u,v\in K$ with $u,v>0$ and let $\ell=\mathrm{lcm}(u,v)$.
Then
\[
f(\ell)=f\!\Big(u\cdot\frac{\ell}{u}\Big)=\frac{\ell}{u}\,f(u)
\quad\text{and}\quad
f(\ell)=f\!\Big(v\cdot\frac{\ell}{v}\Big)=\frac{\ell}{v}\,f(v),
\]
hence $f(u)/u=f(v)/v$, independent of $u,v$.
Write this common ratio as $\alpha_\theta\ge0$.
Therefore $f(n)=\alpha_\theta n$ for all $n\in K$, and
$C_\theta(x)=f(\kappa(x))=\alpha_\theta\,\kappa(x)$ for all $x\in\mathcal X_{\mathrm{fin}}$.
\end{proof}

\begin{corollary}[Asymptotic version with sublinear overhead]
If the compositional compute holds up to a sublinear overhead, i.e., for independent $x,y$,
\[
C_\theta(x\oplus y)=C_\theta(x)+C_\theta(y)+o\!\left(\kappa(x)+\kappa(y)\right),
\]
and the same (A3) assumption holds, then the above proof yields
\[
C_\theta(x)=\alpha_\theta\,\kappa(x)+o\big(\kappa(x)\big)\qquad(\kappa(x)\to\infty).
\]
\end{corollary}


\begingroup
\renewcommand\theproposition{2}
\begin{proposition}[Formal Version]
Let $\mathcal{X}_{\mathrm{fin}}=\{x\in \mathcal{X}:\kappa(x)<\infty\}$. Assume the setting and independence notion of Property~\ref{pro:accuracy_mono}, Property~\ref{pro:accuracy_compo}, and Assumption (A3).
Then there exists $\lambda_\theta\ge 0$ such that for all $x\in \mathcal{X}_{\mathrm{fin}}$ with $0<A_\theta(x)\leq1$,
\[
A_\theta(x)=\exp\!\big(-\lambda_\theta\,\kappa(x)\big).
\]
\end{proposition}


\begin{proof}
Define an equivalence relation $x\sim y$ iff $\kappa(x)=\kappa(y)$. By (A1), if $x\sim y$ then
both $\kappa(x)\le \kappa(y)$ and $\kappa(y)\le \kappa(x)$ hold, hence $A_\theta(x)\ge A_\theta(y)$ and
$A_\theta(y)\ge A_\theta(x)$, so $A_\theta(x)=A_\theta(y)$. Therefore there exists a well-defined map
\[
f:K\to(0,1],\qquad f(n):=A_\theta(x)\ \text{ for any }x\in\mathcal X_{\mathrm{fin}}\text{ with }\kappa(x)=n.
\]
Let $g:K\to\mathbb R_{\ge0}$ be $g(n):=-\log f(n)$.

Fix $u\in K$ and $m\in\mathbb N$. By (A3), choose jointly independent $x_1,\dots,x_m$ with $\kappa(x_i)=u$.
By (A2) and the fixed bracketing,
\[
\kappa(x_1\oplus\cdots\oplus x_m)=mu,\qquad
A_\theta(x_1\oplus\cdots\oplus x_m)=\prod_{i=1}^m A_\theta(x_i)=\bigl(f(u)\bigr)^m.
\]
Hence
\begin{equation}\label{eq:homog}
g(mu)=-\log A_\theta(x_1\oplus\cdots\oplus x_m)=m\,g(u)\qquad(\forall u\in K,\ \forall m\in\mathbb N).
\end{equation}

If $K=\{0\}$ then $A_\theta\equiv1$ and the claim holds with $\lambda_\theta=0$.
Otherwise, let $u,v\in K$ with $u,v>0$ and set $\ell=\mathrm{lcm}(u,v)$.
Applying Eqn.~\ref{eq:homog} twice gives
\[
g(\ell)=g\!\Big(u\cdot\frac{\ell}{u}\Big)=\frac{\ell}{u}\,g(u)
\quad\text{and}\quad
g(\ell)=g\!\Big(v\cdot\frac{\ell}{v}\Big)=\frac{\ell}{v}\,g(v),
\]
so $g(u)/u=g(v)/v$. This ratio is independent of $u,v>0$ in $K$; denote it by $\lambda_\theta\ge0$.

For any $n\in K$, if $n=0$ then $g(n)=0=\lambda_\theta n$; if $n>0$ pick any $u\in K\setminus\{0\}$ and write
$n=\frac{n}{u}\,u$ to get from Eqn.~\ref{eq:homog} that
$g(n)=\frac{n}{u}g(u)=\lambda_\theta n$. Therefore $g(n)=\lambda_\theta n$ for all $n\in K$, i.e.
$f(n)=\exp(-\lambda_\theta n)$, and for any $x\in\mathcal X_{\mathrm{fin}}$,
\[
A_\theta(x)=f\bigl(\kappa(x)\bigr)=\exp\!\big(-\lambda_\theta\,\kappa(x)\big).
\]
\end{proof}

\begin{corollary}[Asymptotic version with sublinear coupling]
If for independent $x,y$ the multiplicativity holds up to a sublinear deviation in the exponent,
\[
\log A_\theta(x\oplus y)=\log A_\theta(x)+\log A_\theta(y)+o\!\big(\kappa(x)+\kappa(y)\big),
\]
and \emph{(A3)} holds, then
\[
\log A_\theta(x)=-\lambda_\theta\,\kappa(x)+o\big(\kappa(x)\big)\quad(\kappa(x)\to\infty),
\]
equivalently $A_\theta(x)=\exp\!\big(-\lambda_\theta \kappa(x)+o(\kappa(x))\big)$.
\end{corollary}

\section{Additional Discussions on Compute Definitions}\label{app:compute-definition}

In Definition~\ref{def:compute}, we adopt the expected number of reasoning tokens as a practical and effective proxy for measuring reasoning compute. Importantly, our conclusions will remain consistent under multiple reasonable alternatives, including model FLOPs and effective compute components that remove verbosity, repetition, and stylistic artifacts from the reasoning traces.

\subsection{Model FLOPs}

Given a fixed model, the \textit{model FLOPs} incurred during the thinking phase are directly proportional to the number of generated reasoning tokens. Concretely, the thinking-phase FLOPs can be written as a model-dependent scaling factor multiplied by the reasoning token length, where the scaling factor is commonly approximated as $2N$, with $N$ denoting the number of model parameters~\citep{ScalingLawsLLM20,Chinchilla22,InferenceScalingLaws}.

In our \textsc{LoRe-Bench} evaluations, replacing reasoning token length with model FLOPs is therefore equivalent to \textit{uniformly rescaling the y-axis in the plots by a constant factor}. This transformation would not change the relative distribution of data points and does not affect any of our experimental conclusions.

\subsection{Effective Compute}

Beyond total reasoning tokens, we also consider \emph{effective compute}, which aims to capture the meaningful reasoning effort by removing systematic overhead from the reasoning process.

Intuitively, verbosity tends to scale with complexity, as more complex problems naturally require longer chains of reasoning and associated artifacts. To examine the relationship between effective compute and total compute, we conduct an additional experiment on \textsc{LoRe-Compo}, where effective compute is derived by applying heuristic string-matching rules to filter out verbosity artifacts. Specifically, we define a set of rules to identify isolated hedging phrases (\textit{e.g.}, \texttt{``Wait, hold on''}, \texttt{``Hmm''}), contentless fillers, and duplicate sentences, and manually verify the results on a subset of samples.

Our analysis reveals \textit{a strong linear relationship} between effective compute and total compute. Notably, Pearson correlation $r$ exceeds $0.8$ across all settings. Here we define the effective-to-total compute ratio as
$
\text{Eff. Ratio} = \frac{\mathbb{E}[C_{\theta}^{\text{eff}}]}{\mathbb{E}[C_{\theta}]},
$
where $C_{\theta}^{\text{eff}}$ and $C_{\theta}$ denote effective compute and total compute.

This indicates that, for a fixed model, replacing total compute with effective compute is also equivalent to rescaling the y-axis in the plots by an constant factor. Empirically, the effective-to-total compute ratio remains stable across both pre- and post-SFT settings. As a result, substituting effective compute for total compute preserves the relative distribution of data points and also does not alter our conclusions.

\begin{table*}[ht!]
\centering
\caption{\textbf{Analysis of Effective and Total Compute on \textsc{LoRe-Compo}.}
We report the Pearson correlation $r$ between effective compute and total compute, and the effective-to-total compute ratio.}
\label{tab:effective_compute}
\renewcommand{\arraystretch}{1.05}
\resizebox{0.65\linewidth}{!}{
\begin{tabular}{@{}cccc@{}}
\toprule
\textbf{Base Model} & \textbf{Method} & \textbf{Pearson $r$} & \textbf{Effective Ratio} \\
\midrule
\multirow{2}{*}{DeepSeek-R1-1.5B}
& \texttt{Base} & 0.869 & 85.9\% \\
& \texttt{SFT-Compo} & 0.880 & 86.5\% \\
\midrule
\multirow{2}{*}{DeepSeek-R1-7B}
& \texttt{Base} & 0.920 & 93.0\% \\
& \texttt{SFT-Compo} & 0.934 & 90.8\% \\
\bottomrule
\end{tabular}
}
\end{table*}
\vspace{-0.5cm}

\section{Additional Details and Results of LoRe-Bench}\label{app:lore-bench-q-examples}

\subsection{Additional Details of LoRe-Mono}

\subsubsection{Example seed questions of LoRe-Mono}

Here we provide one representative seed question example for each domain.

\begin{tcolorbox}[          
  enhanced,
  colback=MorandiLighterBlue,   
  colframe=MorandiLightBlue,
  width=\textwidth, 
  boxsep=1mm,
  title={\normalsize Code - Example seed question},
  fonttitle==\small,
  coltitle=black, 
  center title, 
  toptitle=0.6mm,
  bottomtitle=0.6mm,
  fonttitle=\bfseries\large,
  left=3mm,                 
  right=3mm,                
  top=2mm,                   
  bottom=2mm,
  ]

You are given runnable Python 3.10 code. Execute it exactly as-is in a clean environment (no extra imports).
This is a Code Execution task: run the program, do not rewrite it. The loop counter i is 0-based.
Return only the value of ANSWER (no other text, no formatting).\\
Code:
\begin{lstlisting}[style=mypython,escapeinside=@@]
N = {@\textcolor{red}{\textit{N}}@}; s = {init_state!r}

def f(s, i):
    if len(s) == 0:
        return s
    L = len(s); r = (i % L) + 1; s1 = s[r:] + s[:r]
    trans = str.maketrans({'a':'e','e':'i','i':'o','o':'u','u':'a'})
    return s1.translate(trans)

for i in range(N):
    s = f(s, i)

ANSWER = s
\end{lstlisting}

\end{tcolorbox}

\begin{tcolorbox}[          
  enhanced,
  colback=MorandiLighterBlue,   
  colframe=MorandiLightBlue,
  width=\textwidth, 
  boxsep=1mm,
  title={Math - Example seed question},
  coltitle=black, 
  center title, 
  toptitle=0.6mm,
  bottomtitle=0.6mm,
  fonttitle=\bfseries\normalsize,
  left=3mm,                 
  right=3mm,                
  top=2mm,                   
  bottom=2mm,
  ]
Given an integer \(n=\{\textcolor{red}{N}\}\), consider the order-2 recurrence over integers modulo {M} with an alternating update rule and a mild nonlinear term.
You are given the initial values
\[
x_0 = {x0},\quad x_1 = {x1}.
\]

We update the sequence one step at a time. Let \(t = 1,2,3,\dots,n\) denote the update index, where \(t=1\) is the update that produces \(x_2\) from \((x_1, x_0)\).
At each update \(t\), compute \(x_{{k+1}}\) from \((x_k, x_{{k-1}})\) using the parity of \(t\):

- Define the nonlinear map \(\varphi(z) = (z + 1)^2\). (You may reduce intermediate values modulo {M} at any time.)

- \textbf{Odd step (t odd):}
  \[
  x_{{k+1}} \equiv {A}x_k + {B}x_{{k-1}} + {C}\,\varphi(x_k) \pmod{{{M}}}.
  \]

- \textbf{Even step (t even):}
  \[
  x_{{k+1}} \equiv {A}x_k - {B}x_{{k-1}} + {C}\,\varphi(x_{{k-1}}) \pmod{{{M}}}.
  \]

For clarity, the first two updates are:
\[
\begin{aligned}
    t=1:\;& x_2 \equiv {A}x_1 + {B}x_0 + {C}\,\varphi(x_1) \pmod{{{M}}},\\
    t=2:\;& x_3 \equiv {A}x_2 - {B}x_1 + {C}\,\varphi(x_1) \pmod{{{M}}}.
\end{aligned}
\]

Apply exactly \(n-1\) updates starting from \(x_0, x_1\) to reach \(x_n (n={n})\), and **return \(x_{n}\)** as a single non-negative integer in \([0, {M-1}]\).

Conventions:
- All modular reductions are taken modulo {M} and return a non-negative remainder.
- The alternating rule depends on the **update index** \(t\).
- Output only the integer value of \(x_n\) (no extra text).\
\end{tcolorbox}

\begin{tcolorbox}[          
  enhanced,
  colback=MorandiLighterBlue,   
  colframe=MorandiLightBlue,
  width=\textwidth, 
  boxsep=1mm,
  title={\normalsize Science - Example seed question},
  fonttitle==\small,
  coltitle=black, 
  center title, 
  toptitle=0.6mm,
  bottomtitle=0.6mm,
  fonttitle=\bfseries\large,
  left=3mm,                 
  right=3mm,                
  top=2mm,                   
  bottom=2mm,
  ]
You are modeling a **batch bioreactor** where an enzyme E converts substrate A to product B, but each catalytic event requires a recyclable **cofactor token** C (e.g., NAD$^*$/NADH).
Let \(A_t, B_t, C_t\) be the nonnegative integer counts of A, B, and C **after** completing tick \(t\).
You are given fixed initial counts and a regeneration period:
- \(A_0 = {args.A0}\), \(B_0 = {args.B0}\), \(C_0 = {args.C0}\)
- Regeneration period \(k = {k}\)\\

For **each discrete tick** \(t = 1,2,\dots,n\) (with \(n=\{\textcolor{red}{N}\}\)), apply the following **biochemical rule order**:

1) **Reaction (consumes cofactor)** — if both substrate and cofactor are available:
   - If \(A_{{t-1}} > 0\) **and** \(C_{{t-1}} > 0\), then one catalytic turnover occurs:
     \[
     A_t = A_{{t-1}} - 1,\quad B_t = B_{{t-1}} + 1,\quad C_t = C_{{t-1}} - 1.
     \]
   - Otherwise, no reaction this tick:
     \[
     A_t = A_{{t-1}},\ B_t = B_{{t-1}},\ C_t = C_{{t-1}}.
     \]

2) **Cofactor regeneration (post-reaction)** — models an external respiratory/oxidative cycle returning the cofactor to its usable form at fixed intervals:
   - If \(t \bmod k = 0\), then **after** the reaction stage:
     \[
     C_t \leftarrow C_t + 1.
     \]

All updates are integer and at most \(\pm 1\) per tick (“min/+=1” granularity).  
**Output** the product count \(B_n\) after completing exactly \(n=\{\textcolor{red}{N}\}\) ticks (i.e., after applying the regeneration rule at tick \(n\)).

\end{tcolorbox}

\begin{tcolorbox}[          
  enhanced,
  colback=MorandiLighterBlue,   
  colframe=MorandiLightBlue,
  width=\textwidth, 
  boxsep=1mm,
  title={\normalsize Language - Example seed question},
  fonttitle==\small,
  coltitle=black, 
  center title, 
  toptitle=0.6mm,
  bottomtitle=0.6mm,
  fonttitle=\bfseries\large,
  left=3mm,                 
  right=3mm,                
  top=2mm,                   
  bottom=2mm,
  ]
You are given a letter maze and a number of moves n=\{\textcolor{red}{\textit{N}}\}. The maze is a rectangular grid of letters G with h=\{H\} rows and w=\{W\} columns:\\

\{grid\_block\}\\

Start at the cell (r0, c0) = (\{R0\}, \{C0\}). Build a string S as you move:\\
1) First, write down the starting letter G[r0][c0] into S. (This is done before any moves.)\\
2) Then repeat the following exactly n=\{step\} times (t = 1..n):\\
   • Let (r, c) be your current cell BEFORE moving, and let ch = G[r][c].\\
   • Move one step based on ch (case-insensitive):\\
       – If ch $\in$ \{a, e, i, o, u\} (a vowel): move RIGHT → c ← (c + 1) mod \{W\}\\
       – Otherwise (a consonant): move DOWN → r ← (r + 1) mod \{H\}\\
   • After moving to the destination cell (r, c), append its letter G[r][c] to S.\\
   • Now mutate the grid based on ch (the letter you moved FROM):\\
       – If ch is a vowel: cyclically rotate COLUMN c upward by 1.\\
         (Formally, for all i: G[i][c] ← old G[(i + 1) mod {H}][c].)\\
       – Otherwise (ch is a consonant): cyclically rotate ROW r left by 1.\\
         (Formally, for all j: G[r][j] ← old G[r][(j + 1) mod {W}].)\\

Important: The mutation happens AFTER appending G[r][c] to S, and it affects the grid used for the NEXT iteration. Indices are 0-based and the maze wraps around like a torus.

Thus, after n moves, S has length n + 1 (because the starting letter was included).\\

Let k = \{K\}. Your task is to return the word W made by the LAST k letters of S (in order). Output W as a plain string.

\end{tcolorbox}

\subsubsection{Potential Shortcut Failures}
A basic requirement of \textsc{LoRe-Mono} is that the complexity of question variants increases monotonically with the variant index. However, certain seed questions may violate this requirement. For instance, if answers follow a periodic pattern (\textit{e.g.}, when all even-indexed variants have the answer $1$ and all odd-indexed variants have the answer $0$), a model could exploit prior patterns to guess the correct answer without performing the intended computation. To ensure benchmark reliability, we manually reviewed all variants and excluded those exhibiting periodic answer patterns.

\subsection{Additional Results of LoRe-Mono}
\label{app:lore-add-results}

\subsubsection{Additional Visualization Results}
Visualizations of monotonicity results for the remaining 9 LRMs are provided in Fig.~\ref{fig:appendix_lore_bench_mono_results}.
\\
\begin{figure}[htbp]
  \centering
  \includegraphics[width=0.98\linewidth]{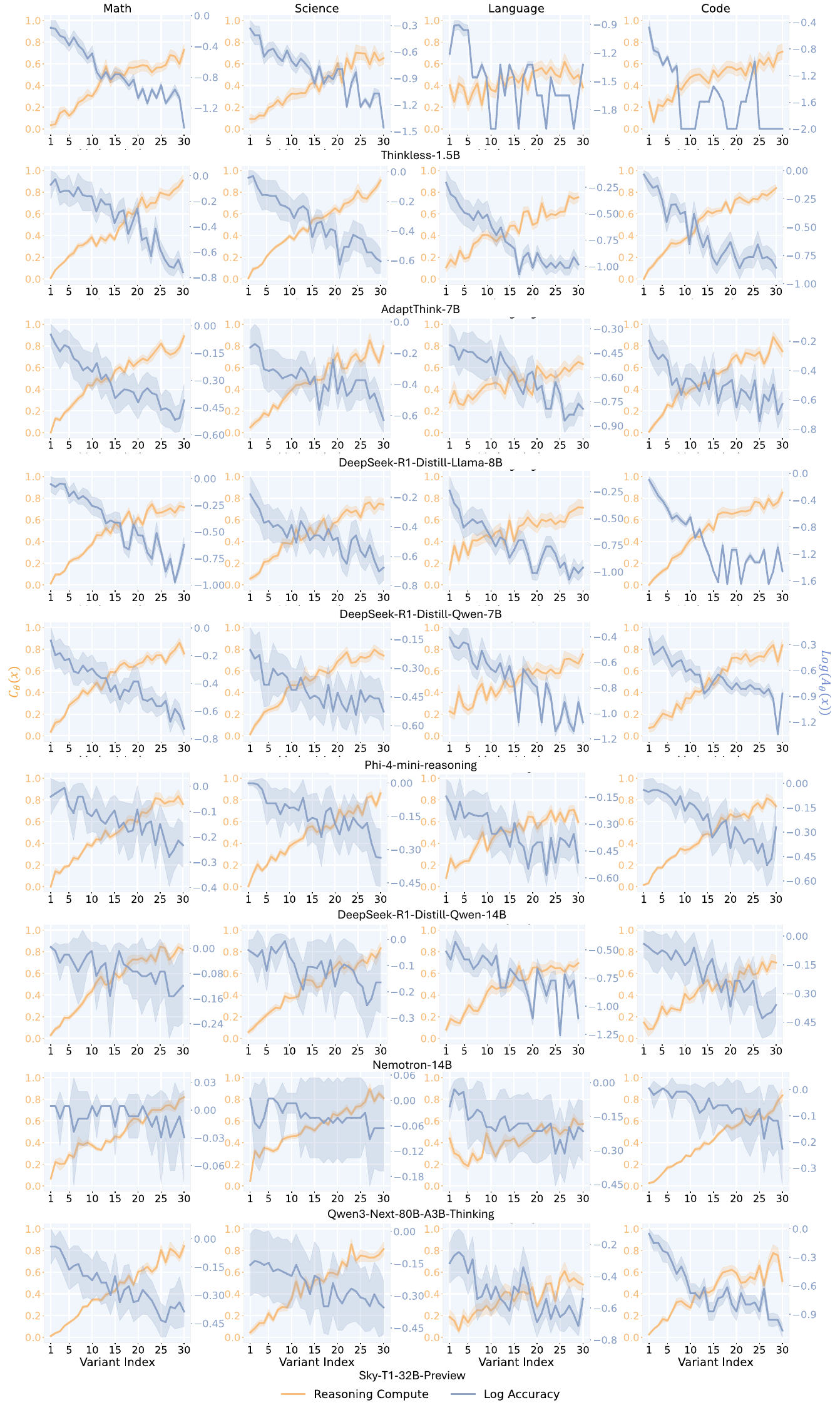}
  \caption{\textbf{Visualizations of Monotonicity Results on Remaining LRMs.}}
  \label{fig:appendix_lore_bench_mono_results}
\end{figure}

\newpage
\subsubsection{"Fail to Think" Case Study}\label{appx:lore-mono-case-study}
We observe that the weakest model, DeepSeek-R1-Distill-Qwen-1.5B, occasionally produces no reasoning tokens on certain code questions. For example, on variant 1 of a seed code question, the generated output is simply \texttt{<think>\textbackslash n</think>}, corresponding to a reasoning length of~$0$.

\begin{tcolorbox}[          
  enhanced,
  colback=MorandiLighterBlue,   
  colframe=MorandiLightBlue,
  width=\textwidth, 
  boxsep=1mm,
  title={\normalsize "Fail to think" - Example model output},
  fonttitle==\small,
  coltitle=black, 
  center title, 
  toptitle=0.6mm,
  bottomtitle=0.6mm,
  fonttitle=\bfseries\large,
  left=3mm,                 
  right=3mm,                
  top=2mm,                   
  bottom=2mm,
  ]
\textcolor{red}{<think>\textbackslash n</think>}\textbackslash n\textbackslash n To solve this problem, we need to simulate the execution of a given Python code and determine the final value assigned to the variable `ANSWER`. The code involves string manipulation and state transitions, which we will simulate step by step.\textbackslash n\textbackslash n\#\#\# Approach\textbackslash n1. **Initialization**:\textbackslash n   - We start with `N = 1` iteration, `s = \textbackslash "01938\textbackslash "`, and a function `f` that processes the string `s`.\textbackslash n\textbackslash n2. **Function Definition**:\textbackslash n   - The function `f` takes a string `s` and processes each character based on its position in the string. The position determines whether the count of '0', '1', '2', or '3' is incremented modulo 5.\textbackslash n\textbackslash n3. **Processing Each Character**:\textbackslash n   - For each character in `s`, we count the occurrences of '0', '1', '2', and '3'..... (\textit{omitted ...})
\end{tcolorbox}

\subsection{Additional Details of LoRe-Compo}
For all composite questions, a fixed connector links the sub-questions, ensuring they are answered in sequential order. An illustrative example is provided below.

\begin{tcolorbox}[          
  enhanced,
  colback=MorandiLighterBlue,   
  colframe=MorandiLightBlue,
  width=\textwidth, 
  boxsep=1mm,
  title={\normalsize Example composite question},
  fonttitle==\small,
  coltitle=black, 
  center title, 
  toptitle=0.6mm,
  bottomtitle=0.6mm,
  fonttitle=\bfseries\large,
  left=3mm,                 
  right=3mm,                
  top=2mm,                   
  bottom=2mm,
  ]
Answer the following questions in order:

Q1. Jim and Martha are standing together at the corner of a rectangular field. Jim walks diagonally across the field. Martha gets to the same location by walking along its length and width. The field is 300 feet wide and 400 feet long. How many feet less than Martha does Jim walk?

Q2. Find all values of $x$ that satisfy the equation $x = \!\sqrt{11-2x} + 4$.
\end{tcolorbox}

\subsection{Additional Results of LoRe-Compo}\label{}
Visualizations of compositionality results for the remaining LRMs are provided in Fig.~\ref{fig:compo-bf-sft}.

\begin{figure}[ht]
    \centering
    \includegraphics[width=1.0\linewidth]{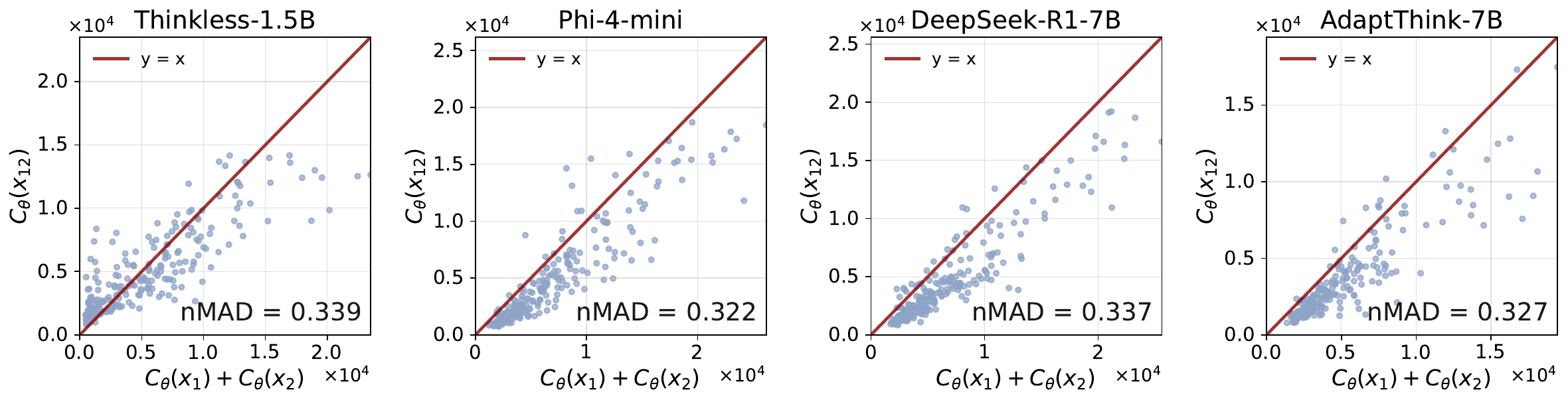}
    \includegraphics[width=1.0\linewidth]{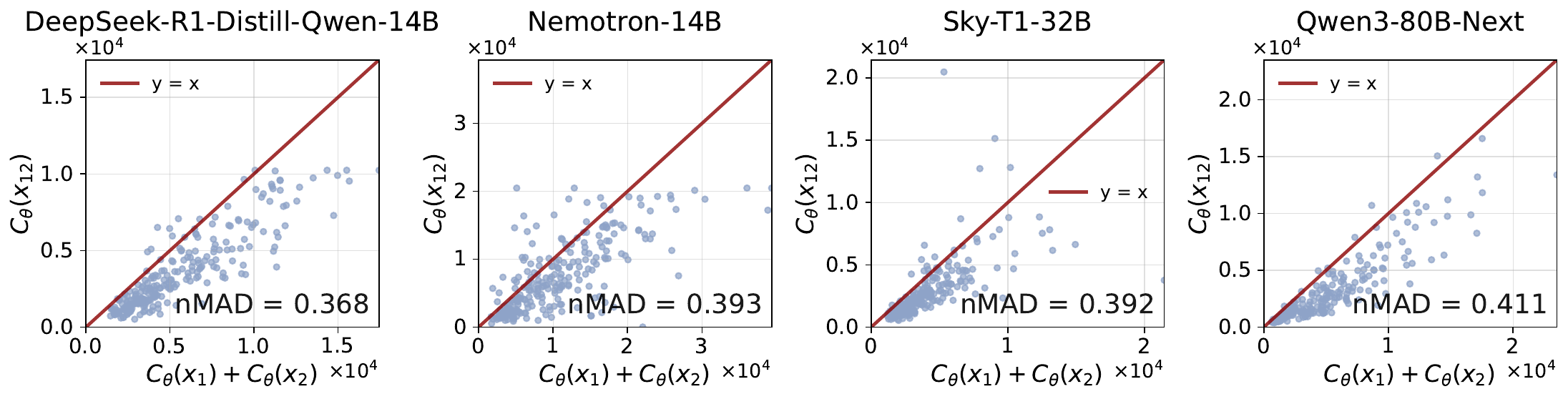}
    \caption{\textbf{Visualizations of Reasoning Compute Compositionality on More LRMs.}}
    \label{fig:compo-bf-sft}
\end{figure}

\section{Experimental Details and Additional Results}\label{appx:lore_bench_additional_results}

\subsection{Computational Budget}\label{compute-budget}

We used 8 NVIDIA L40S 48G GPUs and 8 NVIDIA RTX PRO 6000 Blackwell 96G GPUs for the experiments. For SFT-Compo training, the runtime varies with model size and is approximately 2–6 hours on RTX PRO 6000 GPUs. For inference, the runtime also depends on model size: LoRe-Bench evaluation takes approximately 30 hours for all models, while general reasoning benchmark evaluation takes around 10 hours for all fine-tuned models, both conducted on L40S GPUs.

\subsection{Implementation Details}
\label{app-imple}

Since DeepScaler does not come with predefined categories, we first annotate each question using GPT-4.1-mini to assign it to one of the following categories: Algebra (Prealgebra), Counting \& Probability, Geometry, Number Theory, or Calculus (Precalculus). Based on these annotations, we construct sub-question and composite-question triplets by pairing questions from different categories. 

For SFT, we perform a grid search over learning rates in \{1e-6, 5e-6, 5e-5\}, using a batch size of 8, gradient accumulation of 2, and a warmup ratio of 0.

\subsection{Additional Proprietary Model Analysis}
\label{addi-close-source}

To further demonstrate the generalization of our conclusions beyond open-source LRMs, we conduct an evaluation on a representative proprietary model, \textit{Gemini 3.1 Flash-Lite}. Our experiments reveal two important findings: (1) this proprietary model still fails to exhibit compositionality, and (2) it does not better satisfy monotonicity or compositionality than small-scale models.

Due to API constraints, proprietary systems typically return only \textit{summarized reasoning}, rather than full reasoning traces, which may not accurately reflect the real-world inference cost. A more reliable evaluation and conclusion would require internal measurements from the commercial service providers.

\begin{table*}[ht!]
\caption{\textbf{Monotonicity Results on Proprietary Models.} We report Spearman correlations for reasoning compute and log accuracy on \textsc{LoRe-Mono}. \textit{Lang.} denotes \textit{Language}.}
\label{table:mono-before-sft-gemini}
\centering
\renewcommand{\arraystretch}{1.05}
\setlength{\tabcolsep}{1.7mm}
\resizebox{1\linewidth}{!}{
\begin{tabular}{ccccccccccc}
\toprule[0.95pt]
\multirow{2}{*}[-0.7ex]{\textbf{Models}} 
& \multicolumn{5}{c}{\textbf{Reasoning Compute $\uparrow$}} 
& \multicolumn{5}{c}{\textbf{Log Accuracy $\downarrow$}} \\
\cmidrule(lr){2-6}\cmidrule(lr){7-11}
& \textbf{Math} & \textbf{Science} & \textbf{Lang.} & \textbf{Code} & \textbf{All} & \textbf{Math} & \textbf{Science} & \textbf{Lang.} & \textbf{Code} & \textbf{All}\\
\midrule[0.6pt]

Gemini 3.1 Flash-Lite & 0.976 & 0.916 & 0.969  & 0.991 & \textbf{0.976} & -0.922 & \cellcolor[RGB]{250,235,215}-0.690 & \cellcolor[RGB]{250,235,215}-0.321 & -0.874 & \textbf{-0.935}\\

\bottomrule[0.95pt]
\end{tabular}
}
\end{table*}

\begin{table*}[ht!]
\centering
\caption{\textbf{Compositionality Results on \textsc{LoRe-Compo}.}
We calculate $\mathrm{nMAD}$ for reasoning compute ($C_\theta$) and log accuracy ($\log A_\theta$).}
\label{tab:compo_results_gemini}
\renewcommand{\arraystretch}{1.05}
\resizebox{0.5\linewidth}{!}{
\begin{tabular}{@{}ccc@{}}
\toprule
\textbf{Models} & $\text{nMAD}_{C_\theta} \downarrow$ & $\text{nMAD}_{\log A_\theta} \downarrow$ \\
\midrule
Gemini 3.1 Flash-Lite & 0.343 & 0.490 \\
\bottomrule
\end{tabular}
}
\end{table*}

\subsection{Additional Experimental Results}
\label{app}
In Fig.~\ref{fig:compo-bf-af-sft-compare}, we further compare the reasoning compute compositionality of DeepSeek-R1-Distill-Qwen-7B and DeepSeek-R1-Distill-Llama-8B before and after \texttt{SFT-Compo}. With \texttt{SFT-Compo}, the $\mathrm{nMAD}$ decreases, and the results align more closely with the $y = x$ line compared to their base counterparts.


\subsection{Analysis of Prompt Variations}
\label{prompt-ablation}

In this section, we analyze the sensitivity of our conclusions to prompt variations and perform controlled ablations to rule out potential confounding effects from different prompt formulations.

\begin{figure}[!t]
    \centering
    \includegraphics[width=1.0\linewidth]{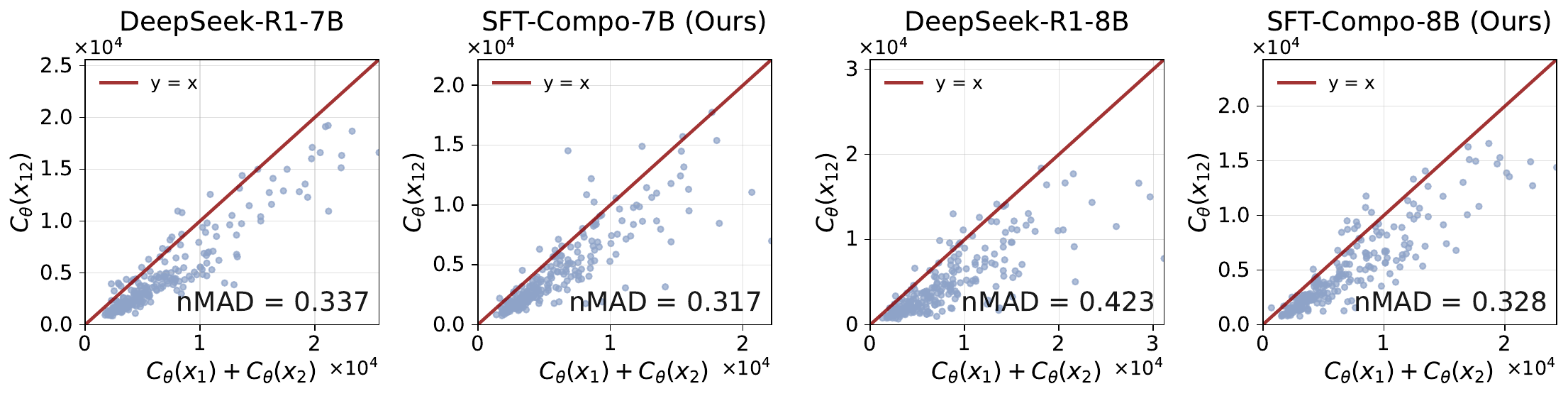}
    \caption{\textbf{Visualizations of Reasoning Compute Compositionality on DeepSeek-R1-Distill-Qwen-7B and DeepSeek-R1-Distill-Llama-8B.}}
    \label{fig:compo-bf-af-sft-compare}
\end{figure}

\subsubsection{Prompt Length Effects}

In addition to the standard prompt from the official technical report that is used in this work, we have conducted additional experiments comparing it with two variants: one adding a brief role description, and another adding detailed reasoning instructions, listed as follows,
\begin{itemize}
    \item \textbf{Standard:} \texttt{Please reason step by step, and put your final answer within \textbackslash\textbackslash boxed\{\}.}
    \item \textbf{Variant A:}
    \texttt{
    You are a language model trained to help users. Your role is to explore questions through thorough and careful reasoning before providing the final solutions. This includes analysis, summary, reflection and iteration for well-considered thinking. Please reason step by step, and put your final answer within \textbackslash\textbackslash boxed\{\}.
    }
    \item \textbf{Variant B:}
    \texttt{
    You are a language model trained to help users. Your role as an assistant involves thoroughly exploring questions through a systematic thinking process before providing the final precise and accurate solutions. This requires engaging in a comprehensive cycle of analysis, summarizing, exploration, reassessment, reflection, backtracing, and iteration to develop well-considered thinking process. Please reason step by step, and put your final answer within \textbackslash\textbackslash boxed\{\}.
    }
\end{itemize}

The results are shown in Tab.~\ref{tab:ablation_sensitivity}. In the \textsc{LoRe-Compo} evaluation, variations in prompt length lead to only minor variations, and $\mathrm{nMAD}_{C_\theta}$ metric used to evaluate compositionality remains consistently large. This indicates that the models still fail to exhibit compositionality, regardless of prompt length change.
\begin{table*}[ht!]
\centering
\caption{\textbf{Prompt Ablation Results.}
We report $\mathrm{nMAD}_{C_\theta}$ under different prompt variants.}
\label{tab:ablation_sensitivity}
\renewcommand{\arraystretch}{1.05}
\resizebox{0.72\linewidth}{!}{
\begin{tabular}{@{}ccccc@{}}
\toprule
\textbf{Models} & \textbf{Standard} & \textbf{Variant A} & \textbf{Variant B} & \textbf{Formatting} \\
\midrule
DeepSeek-R1-1.5B & 0.528 & 0.535 & 0.481 & 0.478 \\
Phi-4-mini    & 0.322 & 0.315 & 0.318 & 0.320 \\
DeepSeek-R1-7B   & 0.337 & 0.338 & 0.339 & 0.344 \\
\bottomrule
\end{tabular}
}
\end{table*}

\subsubsection{Formatting Artifacts}

To eliminate the effect of prompt formatting, we presented the sub-questions in the following format to make them consistent with the composite question.
\begin{itemize}
    \item \textbf{Formatting Variant:} \texttt{Answer the following question. Q1: \{sub-question\}}
\end{itemize}
The results in Tab.~\ref{tab:ablation_sensitivity} show that this modification leads to only minor variations in performance. This suggests that formatting differences are unlikely to be the primary cause of the observed compositionality failures.

\subsubsection{Additional Reasoning Overhead}

Concatenation may force the model to generate extra meta-reasoning to manage multiple tasks, which could potentially lead to $C_\theta(x_1 \oplus x_2) > C_\theta(x_1) + C_\theta(x_2)$. According to the empirical visualizations in Fig.~\ref{fig:compo_quali} and Fig.~\ref{fig:compo-bf-sft}, across all model sizes, we observed that most data points lie in the region where $C_\theta(x_1 \oplus x_2) < C_\theta(x_1) + C_\theta(x_2)$. Therefore, if meta-reasoning were induced, the discrepancy between $C_\theta(x_1 \oplus x_2)$ and $C_\theta(x_1) + C_\theta(x_2)$ would be even larger and the $\mathrm{nMAD}_{C_\theta}$ metric would be worse than currently measured. This implies that the models’ failure to exhibit compositionality may be even more severe than what is reflected in our current results. This further highlights the importance of studying reasoning laws.

\subsection{Additional Evaluations Beyond Mathematics}
\label{sec:commonsense}

To examine whether popular LRMs follow the compositionality beyond the mathematical domain, we evaluate three representative base models on a \textit{widely used commonsense reasoning} benchmark, CommonsenseQA~\citep{CommonsenseQA}.

Specifically, we sampled 500 CommonsenseQA problems and constructed a \textsc{LoRe-Compo-Commonsense} set using the same procedure as in \textsc{LoRe-Compo}. We then measured compositionality using the same metric $\mathrm{nMAD}_{C_\theta}$.
\begin{table*}[ht!]
\centering
\caption{\textbf{Compositionality Results on Commonsense Reasoning.}
We report $\mathrm{nMAD}_{C_\theta}$ on \textsc{LoRe-Compo-Commonsense} and compare with \textsc{LoRe-Compo}.}
\label{tab:commonsense_compo}
\renewcommand{\arraystretch}{1.05}
\resizebox{0.76\linewidth}{!}{
\begin{tabular}{@{}ccc@{}}
\toprule
\textbf{Models} & \textbf{\textsc{LoRe-Compo-Commonsense}} & \textbf{\textsc{LoRe-Compo}} \\
\midrule
DeepSeek-R1-1.5B & 0.475 & 0.528 \\
Phi-4-mini       & 0.479 & 0.322 \\
DeepSeek-R1-7B   & 0.481 & 0.337 \\
\bottomrule
\end{tabular}
}
\end{table*}

Our experiments reveal two important findings: (1) all these models still failed to exhibit compositionality on commonsense reasoning tasks, and (2) several models even perform much worse on non-mathematical reasoning tasks, such as Phi-4-mini and DeepSeek-R1-7B. These results suggest that compositionality failures are not limited to mathematical reasoning, but generalize to broader reasoning domains.

\section{{Qualitative Examples of LoRe-Compo}}

We present an example generated by DeepSeek-Distill-1.5B and after SFT-Compo. For DeepSeek-Distill-1.5B, we report metrics for the question pair and the composite question with accuracies $A_\theta(x_1)=1.0$, $A_\theta(x_2)=1.0$, $A_\theta(x_{12})=0.25$. 
Reasoning compute is $C_\theta(x_1)=2177.25$, $C_\theta(x_2)=3546.50$, and $C_\theta(x_{12}
)=732.13$, showing that the model severely underthinks on the composite question and compromises the accuracy. After SFT-Compo, the composite accuracy improves to $A_{\tilde{\theta}}(x_{12})=0.875$ with reasoning compute as $C_{\tilde{\theta}}(x_{12})=4300.00$. We provide detailed reasoning processes of the two models below.

\begin{tcolorbox}[          
  enhanced,
  colback=MorandiLighterBlue,   
  colframe=MorandiLightBlue,
  width=\textwidth, 
  boxsep=1mm,
  title={\normalsize Example composite question},
  fonttitle==\small,
  coltitle=black, 
  center title, 
  toptitle=0.6mm,
  bottomtitle=0.6mm,
  fonttitle=\bfseries\large,
  left=3mm,                 
  right=3mm,                
  top=2mm,                   
  bottom=2mm,
  ]
Q: Answer the following questions in order:\\
Q1. Pat is to select six cookies from a tray containing only chocolate chip, oatmeal, and peanut butter cookies. There are at least six of each of these three kinds of cookies on the tray. How many different assortments of six cookies can be selected? (Note that cookies of the same type are not distinguishable.)\\
Q2. Let $f(x) = x^{10}+5x^9-8x^8+7x^7-x^6-12x^5+4x^4-8x^3+12x^2-5x-5$.

Without using long division (which would be horribly nasty!), find the remainder when $f(x)$ is divided by $x^2-1$.\\
\end{tcolorbox}

\begin{tcolorbox}[          
  enhanced,
  colback=MorandiLighterBlue,   
  colframe=MorandiLightBlue,
  width=\textwidth, 
  boxsep=1mm,
  title={\normalsize Example reasoning process (DeepSeek-R1-1.5B)},
  fonttitle==\small,
  coltitle=black, 
  center title, 
  toptitle=0.6mm,
  bottomtitle=0.6mm,
  fonttitle=\bfseries\large,
  left=3mm,                 
  right=3mm,                
  top=2mm,                   
  bottom=2mm,
  ]
Q1. To determine the number of different assortments of six cookies, I need to consider the number of ways to select six cookies from three types: chocolate chip, oatmeal, and peanut butter. Since the cookies of the same type are indistinguishable, the order of selection doesn't matter.\\

I can model this problem using combinations with repetition. The formula for the number of ways to choose \( k \) items from \( n \) types is given by:
\[
\binom{n + k - 1}{k}
\]
In this case, \( n = 3 \) (types of cookies) and \( k = 6 \) (cookies to select). Plugging in the values:
\[
\binom{3 + 6 - 1}{6} = \binom{8}{6}
\]
Calculating \( \binom{8}{6} \) gives:
\[
\binom{8}{6} = \frac{8!}{6!(8 - 6)!} = 28 \quad\textcolor{blue}{\text{[Correct Answer]}}
\]
Therefore, there are 28 different assortments of six cookies that can be selected.\\

Q2. To find the remainder when \( f(x) \) is divided by \( x^2 - 1 \) without using long division, I can utilize the Remainder Theorem. The Remainder Theorem states that the remainder of a polynomial \( f(x) \) divided by \( x - a \) is \( f(a) \). However, since the divisor here is \( x^2 - 1 \), which factors into \( (x - 1)(x + 1) \), I need to evaluate \( f(x) \) at the roots of \( x^2 - 1 \), which are \( x = 1 \) and \( x = -1 \).\\

The remainder will be a linear polynomial of the form \( R(x) = ax + b \). To find \( a \) and \( b \), I can set up two equations by plugging in \( x = 1 \) and \( x = -1 \) into \( f(x) \):
1. \( f(1) = a(1) + b \)
2. \( f(-1) = a(-1) + b \) \\

Calculating \( f(1) \):
\[
f(1) = 1^{10} + 5(1)^9 - 8(1)^8 + 7(1)^7 - 1(1)^6 - 12(1)^5 + 4(1)^4 - 8(1)^3 + 12(1)^2 - 5(1) - 5
\]
\[
f(1) = 1 + 5 - 8 + 7 - 1 - 12 + 4 - 8 + 12 - 5 - 5
\]
\[
f(1) = -1
\]

Calculating \( f(-1) \):
\begin{align*}
    f(-1) = (-1)^{10} + 5(-1)^9 - 8(-1)^8 + 7(-1)^7 - 1(-1)^6 - 12(-1)^5 + 4(-1)^4 \\
- 8(-1)^3 + 12(-1)^2 - 5(-1) - 5
\end{align*}
\[
f(-1) = 1 - 5 - 8 - 7 - 1 + 12 + 4 + 8 + 12 + 5 - 5
\]
\[
f(-1) = 13
\]

Now, setting up the equations:
1. \( a + b = -1 \)
2. \( -a + b = 13 \)

Solving these equations:
Adding both equations:
\( 2b = 12 \) => \( b = 6 \)
Substituting \( b = 6 \) into the first equation:
\( a + 6 = -1 \) => \( a = -7 \)\\

Thus, the remainder \( R(x) \) is:
\[
R(x) = -7x + 6 \quad \text{\textcolor{red}{[Incorrect Answer]}}
\] 
</think> \\
\end{tcolorbox}

\begin{tcolorbox}[          
  enhanced,
  colback=MorandiLighterBlue,   
  colframe=MorandiLightBlue,
  width=\textwidth, 
  boxsep=1mm,
  title={\normalsize Example reasoning process (DeepSeek-R1-1.5B SFT-Compo)},
  fonttitle==\small,
  coltitle=black, 
  center title, 
  toptitle=0.6mm,
  bottomtitle=0.6mm,
  fonttitle=\bfseries\large,
  left=3mm,                 
  right=3mm,                
  top=2mm,                   
  bottom=2mm,
  ]
Okay, so I have these two questions to solve. Let me start with the first one.\\

**Q1. Pat is to select six cookies from a tray containing only chocolate chip, oatmeal, and peanut butter cookies. There are at least six of each of these three kinds of cookies on the tray. How many different assortments of six cookies can be selected? (Note that cookies of the same type are not distinguishable.)**\\

Hmm, okay. So Pat is selecting six cookies, and the tray has three types: chocolate chip, oatmeal, and peanut butter. Each type has at least six cookies, so there's no shortage. Since the cookies of the same type are indistinct, we're dealing with combinations with repetition, I think.\\

I remember that the formula for combinations with repetition is C(n + k - 1, k), where n is the number of types and k is the number of items to choose. In this case, n is 3 (since there are three types of cookies) and k is 6 (since Pat is selecting six cookies). So plugging into the formula, it should be C(3 + 6 - 1, 6) = C(8, 6).\\

Wait, let me double-check. So, combinations with repetition: if order doesn't matter and items are indistinct except for type, yes, that's the case here. So, the formula is indeed C(n + k - 1, k). So, n=3, k=6, so 3 + 6 -1 = 8, so C(8,6). Alternatively, C(8,2) since combinations are symmetric, which is 28. So is that the answer?\\

But wait, let me think again. Since each type has at least six cookies, and we're selecting six, does that affect anything? Because if each type has at least six, then when we select six, the maximum number we can take from any one type is six, but since we're only selecting six, the maximum from any type is six, but since we're selecting six in total, the maximum from any type is six, but the minimum is zero.\\

But in this case, since we have at least six of each, so when we choose six, the number of ways is the same as the number of non-negative integer solutions to the equation:\\

x1 + x2 + x3 = 6, where x1, x2, x3 are the number of chocolate chip, oatmeal, and peanut butter cookies respectively.\\

Which is exactly the stars and bars problem, so C(6 + 3 -1, 3 -1) = C(8,2) = 28.\\

So yeah, the answer is 28. That seems right.\\

**Q2. Let \(f(x) = x^{10} + 5x^9 -8x^8 +7x^7 -x^6 -12x^5 +4x^4 -8x^3 +12x^2 -5x -5\).**\\

Without using long division (which would be horribly nasty!), find the remainder when $f(x)$ is divided by $x^2 -1$. Okay, so polynomial division. But $f(x)$ is a degree 10 polynomial, and we're dividing by $x^2 -1$, which is also degree 2. So the remainder should be of degree less than 2, so it should be a linear polynomial, right? So it can be written as ax + b, where a and b are constants.\\

To find a and b, we can use the Remainder Theorem. Since $x^2 -1$ factors as $(x -1)(x +1)$, the remainder when f(x) is divided by $x^2 -1$ is the same as the remainder when f(1) and f(-1) are computed, and then we can set up equations to solve for a and b.\\

Wait, let me recall: if we have f(x) divided by $(x - c)(x - d)$, the remainder is R(x) = e x + f. So, to find R(x), we can evaluate f(c) and f(d), which gives us two equations: $f(c) = e c + f \quad f(d) = e d + f$\\

Then, solving these two equations gives us e and f. In our case, the divisor is $x^2 -1$, which is $(x -1)(x +1)$. So, c =1 and d = -1.\\

\textbf{[To be continued]}
\end{tcolorbox}

\begin{tcolorbox}[          
  enhanced,
  colback=MorandiLighterBlue,   
  colframe=MorandiLightBlue,
  width=\textwidth, 
  boxsep=1mm,
  fonttitle==\small,
  coltitle=black, 
  center title, 
  toptitle=0.6mm,
  bottomtitle=0.6mm,
  fonttitle=\bfseries\large,
  left=3mm,                 
  right=3mm,                
  top=2mm,                   
  bottom=2mm,
  ]
Therefore, we can compute f(1) and f(-1), set up the equations:\\
$f(1) = a(1) + b = a + b$

$f(-1) = a(-1) + b = -a + b$

Then, solving these two equations will give us a and b.\\

So, let me compute f(1) and f(-1). First, f(1):

$f(1) = (1)^10 + 5*(1)^9 -8*(1)^8 +7*(1)^7 - (1)^6 -12*(1)^5 +4*(1)^4 -8*(1)^3 +12*(1)^2 -5*(1) -5$\\

Simplify each term:

$1 + 5 -8 +7 -1 -12 +4 -8 +12 -5 -5$\\

Let me compute step by step: Start with 1.

$1 +5 =6 \quad 6 -8 =-2 \quad -2 +7=5 \quad 5 -1=4 \quad 4 -12=-8 \quad -8 +4=-4$

$-4 -8=-12 \quad -12 +12=0 \quad 0 -5=-5 \quad -5 -5=-10$

So $f(1) = -10$. Similarly, compute $f(-1)$:

$f(-1) = (-1)^10 +5*(-1)^9 -8*(-1)^8 +7*(-1)^7 - (-1)^6 -12*(-1)^5 +4*(-1)^4 -8*(-1)^3 +12*(-1)^2 -5*(-1) -5$\\

Simplify each term:

$(-1)^10 =1\quad 5*(-1)^9=5*(-1)=-5\quad -8*(-1)^8=-8*(1)=-8$

$7*(-1)^7=7*(-1)=-7\quad - (-1)^6= -1\quad -12*(-1)^5=-12*(-1)=12$

$4*(-1)^4=4*(1)=4 \quad -8*(-1)^3=-8*(-1)=8 \quad 12*(-1)^2=12*(1)=12 \quad -5*(-1)=5$

-5 remains as is.\\

So let's write all terms: $1 -5 -8 -7 -1 +12 +4 +8 +12 +5 -5$\\

Let me compute step by step:\\

Start with 1.\\

$1 -5 =-4 \quad -4 -8=-12 \quad -12 -7=-19= \quad -19 -1=-20$

$-20 +12=-8 \quad -8 +4=-4 \quad -4 +8=4 \quad 4 +12=16 \quad 16 +5=21 \quad 21 -5=16$

So $f(-1)=16$. Therefore, we have: $f(1)=a + b = -10\quad f(-1)= -a + b =16$\\

So now, we have the system of equations: 1) $a + b = -10$ 2) $-a + b =16$. Let me solve this.\\

Add equations 1 and 2: $(a + b) + (-a + b) = (-10) +16$

Simplify: $0a + 2b =6$ So $2b=6 => b=3$. Then, plug b=3 into equation 1: $a +3= -10 => a= -13$. So the remainder is $ax + b = -13x +3$\\

Therefore, the remainder when f(x) is divided by $x^2 -1$ is $-13x +3$. Let me double-check my computations for f(1) and f(-1), just to make sure I didn't make any arithmetic mistakes.\\

**Recomputing f(1):**

$1 +5=6 \quad 6-8=-2 \quad -2+7=5 \quad 5-1=4 \quad 4-12=-8 \quad -8+4=-4$

$-4-8=-12 \quad -12+12=0 \quad 0-5=-5 \quad -5-5=-10$\\

Yes, that's correct. **Recomputing $f(-1)$:**
$1 -5= -4 \qquad
-4 -8= -12 \qquad
-12 -7= -19 \qquad
-19 -1= -20$ $
-20 +12= -8 \qquad
-8 +4= -4 \qquad
-4 +8=4 \qquad
4 +12=16 \qquad
16 +5=21 \qquad
21 -5=16 \qquad
$
Yes, that's correct.\\

So the equations are: $a + b = -10 \quad -a + b =16$\\

\textbf{[To be continued]}

\end{tcolorbox}

\begin{tcolorbox}[          
  enhanced,
  colback=MorandiLighterBlue,   
  colframe=MorandiLightBlue,
  width=\textwidth, 
  boxsep=1mm,
  fonttitle==\small,
  coltitle=black, 
  center title, 
  toptitle=0.6mm,
  bottomtitle=0.6mm,
  fonttitle=\bfseries\large,
  left=3mm,                 
  right=3mm,                
  top=2mm,                   
  bottom=2mm,
  ]

Adding them gives $2b=6 => b=3$, then a= -13. So the remainder is -13x +3. I think that's correct.\\

**Final Answer**\\

Q1: \boxed{28} \textcolor{blue}{[Correct Answer]}\\

Q2: \boxed{-13x + 3} \textcolor{blue}{[Correct Answer]}\\

</think>\\
\end{tcolorbox}

\section{Licenses for Existing Assets}\label{asset-licenses}

\subsection{Model Assets}
We utilize a set of publicly available models in our work. The DeepSeek-R1-Distill series models (Qwen-1.5B, Qwen-7B, Qwen-14B, and Llama-8B) are released under the MIT License and are distilled from the Qwen-2.5 series (Apache License 2.0) or Llama-3.1-8B-Base (Llama 3.1 Community License). Additionally, Phi-4-mini-reasoning is released under the MIT License. OpenReasoning-Nemotron-14B is governed by the Creative Commons Attribution 4.0 International License (CC-BY-4.0), with additional Apache License 2.0 information. Sky-T1-32B-Preview, Qwen3-Next-80B-A3B-Thinking, and Thinkless-1.5B-RL-DeepScaleR are released under the Apache License 2.0. AdaptThink-7B-delta0.05 is released under the MIT License. We comply with all respective license terms in our use of these models.

\subsection{Data Assets}
We employ publicly available datasets in our experiments. GSM8K, OlympiadBench, DeepScaler, and CommonsenseQA are released under the MIT License, which permits unrestricted use, modification, and redistribution. MATH500 is derived from the MATH dataset, which is also released under the MIT License. AIME24 and AIME25 are copyrighted by the Mathematical Association of America. The AMC23 dataset does not have an explicitly specified license, so we treat it as having an unspecified license and exercise caution in its usage. We ensure full compliance with the respective license terms of all datasets used.



\end{document}